\documentclass[accepted]{uai2026} % after acceptance, for a revised version; 
\usepackage[american]{babel}
\usepackage{natbib} % has a nice set of citation styles and commands
\usepackage{mathtools} % amsmath with fixes and additions
\usepackage{booktabs} % commands to create good-looking tables
\usepackage{tikz} % nice language for creating drawings and diagrams
\usepackage{multirow}

\usepackage{amssymb}
\usepackage[utf8]{inputenc} % allow utf-8 input
\usepackage[T1]{fontenc}    % use 8-bit T1 fonts
\usepackage{hyperref}       % hyperlinks
\usepackage{url}            % simple URL typesetting
\usepackage{amsfonts}       % blackboard math symbols
\usepackage{nicefrac}       % compact symbols for 1/2, etc.
\usepackage{microtype}      % microtypography
\usepackage{graphicx}
\usepackage{doi}
\usepackage{amsmath}
\usepackage{physics}
\usepackage{wrapfig}
\usepackage{bm}
\usepackage{algorithm,algpseudocode}
\usepackage{xcolor}
\usepackage{forest}
\usepackage{tcolorbox}
\usepackage{tikz-cd}
\usepackage{enumitem}
\usepackage{graphicx}
\usepackage{subcaption}

\newenvironment{proof}{\paragraph{Proof:}}{\hfill$\square$}

\newtheorem{theorem}{Theorem}

\newtheorem{lemma}{Lemma}

\newtheorem{assumption}{Assumption}
\newtheorem{remark}{Remark}

\title{Provably Efficient Reinforcement Learning in Continuous-Time Episodic MDPs with Poisson Decision Epochs}

\author[1]{Kenny~Guo}
\author[2]{Valentio~Iverson}
\author[3]{Sahan~Wijetunga}
\author[3]{William~Chang}
\affil[1]{%
    Department of Economics\\
    Yale University\\
    New Haven, Connecticut, USA
}
\affil[2]{%
    Department of Computer Science\\
    Massachusetts Institute of Technology\\
    Cambridge, Massachusetts, USA
}
\affil[3]{%
    Department of Mathematics\\
    University of California, Los Angeles\\
    Los Angeles, California, USA
}
  
\begin{document}
\maketitle
\begin{abstract}
Many real-world reinforcement learning (RL) problems evolve in continuous time, where decisions occur at irregular, event-driven intervals rather than at fixed discrete steps. We study episodic continuous-time Markov Decision Processes (MDPs) in which decision epochs are governed by a homogeneous Poisson process and the reward and transition dynamics vary smoothly over time. We consider both a fixed number of jumps per episode and a fixed time budget with a random number of Poisson decision epochs. Under a Lipschitz continuity assumption in time, we exploit local smoothness through discretization and extend both UCRL~\cite{auer2006logarithmic} and Q-learning~\cite{jin2018q} to this setting, proving $\tilde{\mathcal{O}}(T^{2/3})$ regret bounds for both model-based and model-free algorithms. Finally, we establish matching $\tilde{\Omega}(T^{2/3})$ minimax lower bounds, showing that the rate is optimal up to logarithmic factors. These results provide the first tight regret guarantees for Lipschitz-smooth continuous-time episodic MDPs with Poisson decision epochs.
\end{abstract}

\section{Introduction} 
Markov decision processes (MDPs) provide a foundational framework for sequential decision-making under uncertainty. Classical analyses assume finite state and action spaces, as well as discrete time jumps across episodes, enabling the design of reinforcement learning algorithms with provable regret guarantees (e.g. UCRL \cite{auer2006logarithmic, auer2008near}, Q-learning \cite{jin2018q}). However, many real-world applications inherently involve continuous domains. For example, in robotics, control actions lie in continuous ranges \cite{devo2021enhancing, talbot2025continuous}; in recommendation systems, states may be represented in high-dimensional continuous features \cite{steck2021deep, rim2025cyclic}; and in queuing or resource allocation problems \cite{fawaz2021deep, zhang2021deep}, time is often treated as a continuous variable. Extending reinforcement learning algorithms to these continuous settings presents both conceptual and technical challenges.

The principal difficulty lies in the exploration-exploitation tradeoff in continuous domains. Uniform exploration is infeasible, yet exploiting structure is essential to achieve nontrivial guarantees. One wide studied structural assumption is a linear structure on the MDPs parameters \cite{jin2020provably}. This assumes that there is a feature vector $\phi(x,a) \in \mathbb{R}^d$ for every state action pair that is known and that reward and transition dynamics are given by an inner product of $\phi(x,a)$ with some unknown vector that only depends on the step. This, however, can be restrictive in certain applications. 

A more general condition often times studied in the context of bandits is Lipschitz continuity \cite{bubeck2011lipschitz, wang2020towards, feng2022lipschitz}. The study of Lipschitz continuity in online decision-making has a long history in multi-armed bandits (MABs), where the actions are indexed by a metric space and the reward function is Lipschitz continuous with respect to this metric, with strong regret guarantees on the order of $\mathcal{O}(T^{\frac{2}{3}})$ possessed by the zooming algorithm and related approaches \cite{podimata2021adaptive, kleinberg2019banditsexpertsmetricspaces} which (adaptively) refine discretizations of the action space. Extending these ideas to MDPs introduces additional layers of complexity: state-dependent Lipschitz structure, nontrivial horizon dynamics, and the need to propagate estimation errors through Bellman recursions. Recent work has begun to address these challenges through model-based and model-free algorithms on Lipschitz state and action spaces \cite{pirotta2015policy, metelli2024performance, asadi2018lipschitz}, but many open questions remain regarding tight theoretical regret bounds, adaptive discretization strategies, continuous-time formulations, and handling unknown smoothness parameters. 

Our work builds on several key research directions including reinforcement learning with indirect feedback models, advances in linear MDPs, randomization and active learning techniques for exploration, and Lipschitz continuous reinforcement learning. A comprehensive overview of the broader related work is in Appendix \ref{sec:related}.

\paragraph{Our Contributions}
In this paper, we formulate and analyze a \emph{continuous-time} extension of episodic tabular MDPs in which the episode horizon is indexed by a continuous time variable, decision epochs are induced by a homogeneous Poisson process with rate $\lambda$, and the reward and transition dynamics are Lipschitz-smooth with respect to time.

The homogeneous Poisson process is a natural abstraction for \emph{event-driven} decision-making. In many queueing, service, and resource-allocation systems, decisions are not made at fixed deterministic times, but are instead triggered by exogenous events such as arrivals, service completions, or user requests. Over short time horizons, these events are often well approximated as occurring independently at a constant rate, leading to exponentially distributed inter-event times and a memoryless decision structure.\footnote{If the homogeneous rate $\lambda$ is unknown, it can be estimated from observed inter-arrival times $\tau_1,\ldots,\tau_N$ via the maximum likelihood estimator $\hat{\lambda}=N/\sum_{i=1}^N \tau_i$. Standard concentration bounds for sums of exponential RVs can then be used to construct high-probability confidence intervals for $\lambda$. Extending the analysis to non-homogeneous or state-dependent arrival rates is an important direction for future work.}

The smoothness assumption complements this event-driven model by enabling local generalization across time. When rewards and transition dynamics vary smoothly with the decision epoch, observations collected at one time remain informative about nearby times. This allows the learner to discretize the continuous-time horizon adaptively, rather than treating each time point independently, while still preserving statistical efficiency.

Our contributions are four-fold.

First, in Section~\ref{sec:preliminary}, we formalize the continuous episode-horizon MDP framework under Lipschitz continuity, in which each episode consists of a \emph{fixed number of decision epochs} $H>0$ generated by a homogeneous Poisson process.

Second, in Section~\ref{sec4.1:ucrl}, we present a model-based algorithm extending UCRL \cite{auer2006logarithmic} via time discretization and confidence sets adapted to the fixed-jump continuous setting, and prove an $\mathcal{O}(T^{2/3})$ regret bound. The analysis controls the randomness in when jumps occur, by deriving a high-probability bound on the physical time of the $H$-th jump across all episodes, which reduces the infinite horizon to a bounded window suitable for discretization.

Third, in Section~\ref{subsec:q-learning}, we develop a simpler, model-free continuous-time extension of Q-learning \cite{jin2018q} for the fixed-jump setting, proving the same $\mathcal{O}(T^{2/3})$ regret rate and showing that smoothness-based discretization can be combined with optimistic Q-learning updates in continuous domains. In Section~\ref{subsec:fixed-time}, we introduce a complementary formulation in which the episode instead ends once a \emph{fixed time budget} $H>0$ is exhausted, so that the number of jumps is random, and highlight how this changes the value-function representation once the time axis is discretized into bins: the fixed-jump setting keeps a separate table of value estimates for each combination of jump count and time bin, whereas here a single table indexed by state and time bin suffices. We adapt the Q-learning algorithm to this setting by replacing the random number of jumps with a high-probability effective horizon, and the analysis goes through as before, again yielding an $\mathcal{O}(T^{2/3})$ rate.

Finally, we establish matching $\widetilde\Omega(T^{2/3})$ minimax lower bounds for \emph{both} formulations, using an information-theoretic construction adapted from the discrete-time literature that reduces the problem to a Lipschitz contextual bandit over time bins. For the fixed-jump formulation (Appendix~\ref{appendix:fixed_jump_lower_bound}) the bound is $\widetilde\Omega\big(L^{1/3}\lambda^{-1/3} H T^{2/3}\big)$, and for the fixed-time formulation (Appendix~\ref{sec:lower_bound_construction}) it is $\widetilde\Omega\big(S L^{1/3}(\lambda H T)^{2/3}\big)$, where the two differ through the effective number of decisions available in each episode ($H$ vs.\ $\lambda H$), together with an additional $S$ factor in the fixed-time bound. Together, these show the $T^{2/3}$ exponent is unavoidable in both settings.

 \section{Preliminary and Problem Statement}\label{sec:preliminary}
 We study a learner facing a $T$-episodic, tabular, continuous-time MDPs, where $\mathcal{S}$ is the (finite) state space and $\mathcal{A}$ is the (finite) action space of the leader. We let $|\mathcal{S}| = S$ and $|\mathcal{A}| = A$. We let the episode time layers range from $h \in [0, \infty)$. In this way, we can view each time layer, state, action triple as $(h, x, a) \in [0, \infty) \times \mathcal{S} \times \mathcal{A}$. 
 
 Each episode, we suppose the learner starts at layer $h =0$ and some fixed state $x_0$. At some state $x$ in layer $h$, the learner takes an action $a$, receives reward $r_h(x,a) \in [0,1]$, and transitions to some next state governed by the distribution $P_h(\cdot|x,a)$. Each episode, the learner makes a \emph{fixed} number $H$ jumps to different states in layers across $[0,\infty)$ where the \emph{layer} of the next state is sampled via a homogeneous Poisson process. Let $h(n)$ be the time at the $n$-th jump, so that $h(n+1) - h(n) \sim \mathrm{Exp}(\lambda)$ for some fixed, known $\lambda$.

 To make learning a continuous space from sampling tractable, we make the following Lipschitz assumption on the rewards and transitions for corresponding state action pairs across distinct layer values of $h$. 
\begin{assumption}\label{assumption:lipschitz}
     There exists some constant $L>0$ such that for all $(x,a) \in \mathcal{S}\cross\mathcal{A}$ and for some $h, h' \in [0, \infty)$, it holds that
      \begin{align}
     |r_h(x, a) - r_{h'}(x,a)| & \leq L|h - h'|,\\
     \norm{P_h(\cdot | x, a) - P_{h'}(\cdot | x, a)}_1 &\leq L|h-h'|.
 \end{align}
 \end{assumption}
 \paragraph{Bellman Equations.} Let $\pi:[0, \infty) \times \mathcal{S} \rightarrow \mathcal{A}$ be a policy of the learner that maps states and layers to actions. For some policy $\pi$, we recursively define the state value, or $V$-value, function and the state-action, or $Q$-value, function using the following Bellman equations for some $(h, x, a)$ and $n\in \{0, 1, \ldots, H-1\}$:
\begin{align}
    &Q_{n, h}^\pi (x, a) \\
    &= r_h(x,a) + \underset{\substack{h' \sim \mathrm{Exp}(\lambda) \\ x' \sim P_h(\cdot \mid x, a)}}{\mathbb{E}}[V_{n+1, h+h'}^\pi (x')]\\
     &= r_h(x,a) + \int_0^\infty \lambda e^{-\lambda h'}\underset{x' \sim P_h(\cdot |x ,a)}{\mathbb{E}}[V_{n+1, h+h'}^\pi (x')]dh'\label{eq:timebellman}\\
 % =  r_h(x,a) + &\int_0^\infty \lambda e^{-\lambda h'}\underset{x' \sim P_h(\cdot |x ,a)}{\mathbb{E}}[\max_{a'} Q_{n+1, h+h'}^\star (x', a')]dh' \\
     &V_{n,h}^\pi(x) = \underset{{a\sim \pi(\cdot|h,x)}}{\mathbb{E}}Q_{n,h}^\pi(x,a),
\end{align}
where we initialize the last step $V^\pi_{H, h}(x) :=0$ for any $h, x$. Note that based on the Lipschitz condition, for each $a' \in A$, $Q_{n+1, h'}^\pi (x', a')$ is bounded and continuous in $h'$ (induction can show this for all $n$). Thus, the integral in the Bellman equation is integrable and the definitions are valid. From the definition, $V^{\pi}_{0,0}(x_0)$ intuitively captures the expected sum of rewards following $\pi$ starting at the initial state $x_0$ at layer $h=0$.

For each episode $t \in [T]$, the learner decides on a policy $\pi_t$, with the goal of minimizing the \emph{regret}, $R_T$, defined as:
$$
R_T := \sum_{t=1}^T  V^{\pi_\star}_{0,0}(x_0) - V^{\pi_t}_{0,0}(x_0),
$$
where we define $\pi_\star \in \text{argmax}_\pi V^{\pi}_{0,0}(x_0)$ to be (one of) the optimal policies.

\section{Main Algorithms and Results}

This section presents our two main algorithms and their regret guarantees. The construction and analysis follow a common three-step roadmap, which we outline here.

\begin{enumerate}[leftmargin=*, label=\textbf{Step \arabic*.}]
  \item \textbf{Bounding the relevant time interval.} Since each jump length is $\mathrm{Exp}(\lambda)$, a union bound shows all $HT$ jumps across all episodes lie in $[0, M]$ with high probability, for $M = \frac{H}{\lambda}\ln\frac{HT}{\delta}$. This reduces the infinite horizon to a finite domain.
  \item \textbf{Discretization of time.} Partition $[0,M]$ into $K = M/\gamma$ bins of width $\gamma$. We can thus pool samples within same bins. By Lipschitz continuity, pooled samples represent the true kernel at any point in the bin up to an $O(L\gamma)$ approximation error.
  \item \textbf{Balancing estimation error against Lipschitz bias.} Within each bin, statistical estimation error decays as $O(m^{-1/2})$ with the number of samples $m$. The Lipschitz approximation bias is $O(L\gamma)$. Choosing $\gamma \sim T^{-1/3}$ equalizes both terms and yields the $T^{2/3}$ rate.
\end{enumerate}

For the fixed-time formulation (Section~\ref{subsec:fixed-time}), the analysis additionally requires a high-probability bound on the random number of Poisson jumps within each episode, provided by Lemma~\ref{lemma:jumpcount_bound}.

\subsection{Continuous-Time UCRL Algorithm}\label{sec4.1:ucrl}

Upper Confidence Reinforcement Learning (UCRL) \cite{auer2006logarithmic, auer2008near} is a model-based algorithm for the finite-horizon, tabular MDP that employs the "optimism principle'' by constructing high-confidence estimates of the MDP parameters. For transitions, UCRL computes empirical distributions from samples and uses them as centers of shrinking confidence sets $C_{t,\delta}$, guaranteeing that the true transition kernel $P^\star$ lies in $C_{t,\delta}$ for all $t$ with probability at least $1-\delta$.

From this confidence set, for each episode $t$, UCRL chooses the optimistic policy
\begin{equation}\label{eq:ucrlpolicy}
    \pi_t = \text{argmax}_\pi \max_{P \in C_{t, \delta}} V^\pi_P(x_0)
\end{equation}
where the (tabular) value functions are computed using an optimistic transition model
$\widetilde{P}_t = \arg \max_{P\in C_{t, \delta}} V^\star_P(x_0)$ (and upper-confidence reward estimates if stochastic). While computationally difficult, UCRL achieves an optimal $\mathcal{O}(\sqrt{T})$ regret in traditional finite-horizon MDPs, and provides a foundation for our analysis of the more implementation-friendly Q-learning algorithm.

We now analyze our continuous-time extension of UCRL. The full algorithm is given in Algorithm \ref{algo:time}. In the next sections, we describe the key details that yield the $\mathcal{O}(T^{2/3})$ guarantee in Theorem \ref{thm:ucrlregret}.

\textbf{Step 1.} As aforementioned, our first step is to show that with high probability it suffices to restrict the analysis to a finite interval $[0,M]$, so that across all $T$ episodes the endpoint of the $H$-th jump lies in $[0,M]$.

\begin{lemma}\label{lemma:horizon_M}
When $M = \frac{H}{\lambda} \ln \left( \frac{HT}{\delta} \right)$, then all $T$ episodes will terminate, i.e. have the $H$-th jump, within $[0,M]$ with probability at least $1 - \delta$. 
\end{lemma}
\begin{proof}
It suffices to show that with high probability each of the $HT$ jumps across all episodes has length at most $\frac{1}{\lambda}\ln\left(\frac{HT}{\delta}\right)$. Since $h' \sim \mathrm{Exp}(\lambda)$,
\[
P\left(h' > \frac{1}{\lambda} \ln \left(\frac{HT}{\delta}  \right)\right) = \frac{\delta}{HT}.
\]
A union bound over the $HT$ jumps implies the probability that any jump exceeds this length is at most $\delta$; taking the complement gives the result.
\end{proof}

\textbf{Step 2.} With a high-probability finite horizon, we can apply discretization. Partition $[0,M]$ into $K = \frac{H}{\lambda \gamma}\log\frac{HT}{\delta}$ disjoint bins of width $\gamma$, and call the set of bins $\mathcal{K}$. Demarcate each bin by its midpoint; let $k$ denote a midpoint and, abusing language, we refer to the corresponding bin by $k$. For any $h \in [0,M]$, let $k(h)$ be the bin containing $h$. We then form empirical estimates of rewards and transitions on this discretized set.

\paragraph{Confidence set for transitions.}
At each episode, for each bin-state-action triple $(k,x,a)$ and next state $x'$, we compute the maximum likelihood estimate
\begin{equation}
\widehat{P}_{t,k}(x'|x,a) = \frac{N_t(k,x,a,x')}{\max\{1, N_t(k,x,a)\}},
\end{equation}
where $N_t(k,x,a,x')$ is the number of times up to episode $t$ that $(x,a)$ \emph{in bin $k$} is visited and transitions to $x'$ in the next jump, and $N_t(k,x,a)$ is the number of times up to episode $t$ that $(x,a)$ in bin $k$ is visited. We then define, for $\delta \in (0,1)$,
\begin{equation}\label{eq:transconfset}
\begin{split}
     C_{t, \delta} := \{P_h(\cdot|x,a) : ||\widehat{P}_{t,k(h)}(\cdot|x,a)
- P_h(\cdot|x,a)||_1 \\
\leq B_\delta(N_t(k,x,a))\}
\end{split}
\end{equation}
where $B_\delta(N_t(k,x,a)) = S\sqrt{\frac{2}{N_t(k,x,a)}\log\frac{1}{\delta}}+L\gamma.$

\begin{lemma}\label{lemma:confidenceset}
    With probability at least $1-\delta$, the true transitions $P^\star_h(\cdot|x,a) \in C_{t, \delta}$ for all $t \in [T]$.
\end{lemma}
\begin{proof}
Pick any bin and fix $(x,a)$ within that bin. For each next state $x'$, let $y_1,\dots,y_n\in\{0,1\}$ be indicators of whether $x'$ is visited on the next jump, with means $P_1(x'),\dots,P_n(x')$. Then $y_i-P_i(x')$ is $1$-subgaussian with mean $0$. By \cite{lattimore2020bandit} (Corollary 5.5), for $\epsilon=\sqrt{\frac{2}{n}\log\frac{1}{\delta}}$,
\[
P\left(\frac{1}{n}\sum_{i=1}^n (y_i-P_i(x')) \ge \epsilon\right)
\le
\exp\left(-\frac{n\epsilon^2}{2}\right)=\delta,
\]
and hence with probability at least $1-\delta$,
\[
\left|\frac{1}{n}\sum_{i=1}^n (y_i-P_i(x'))\right|
\le
\sqrt{\frac{2}{n}\log\frac{1}{\delta}}.
\]
Writing $\widehat{P}(x'|x,a)=\frac{1}{n}\sum_{i=1}^n y_i$, we obtain
\begin{align*}
&\left|\widehat{P}(x'|x,a)-P_h(x'|x,a)\right| \\
&\le
\left|\frac{1}{n}\sum_{i=1}^n (y_i-P_i(x'))\right|
+
\left|\frac{1}{n}\sum_{i=1}^n P_i(x')-P_h(x'|x,a)\right| \\
&\le
\sqrt{\frac{2}{n}\log\frac{1}{\delta}}
+
\frac{1}{n}\sum_{i=1}^n |P_i(x')-P_h(x'|x,a)|.
\end{align*}
Summing over $x'$ and using Lipschitz continuity within the bin yields
\[
\norm{\widehat{P}(\cdot| x, a) - P_h(\cdot | x, a)}_1
\le
S\sqrt{\frac{2}{n}\log\frac{1}{\delta}} + L\gamma,
\]
which implies $P^\star_h(\cdot|x,a)\in C_{t,\delta}$.
\end{proof}

\paragraph{Confidence set for rewards.}\label{lemma:confsetrewards}
We estimate rewards analogously. Fix a bin and a state-action pair, and let $X_1,\dots,X_n$ be the observed rewards with corresponding means $r_1,\dots,r_n$. Since $X_i-r_i$ is $1$-subgaussian,  \cite{lattimore2020bandit} (Corollary 5.5) gives
\[
\left| \frac{1}{n} \sum_{i = 1}^n (X_i - r_i) \right|
\le
\sqrt{\frac{2}{n} \log \frac{1}{\delta}}.
\]
For any reward $r$ realized in this bin,
\begin{align*}
\left| \frac{1}{n} \sum_{i = 1}^n X_i - r \right|
&\le
\left| \frac{1}{n} \sum_{i = 1}^n (X_i - r_i) \right|
+
\left| \frac{1}{n} \sum_{i = 1}^n r_i - r \right| \\
&\le
\sqrt{\frac{2}{n} \log \frac{1}{\delta}}
+
\frac{1}{n} \sum_{i = 1}^n |r_i - r| \\
&\le
\sqrt{\frac{2}{n} \log \frac{1}{\delta}} + L \gamma,
\end{align*}
so with probability at least $1-\delta$ the true reward $r$ lies in
\begin{equation}\label{eq:rewconfset}
\left[ \frac{1}{n} \sum_{i = 1}^n X_i - \Delta,\ \frac{1}{n} \sum_{i = 1}^n X_i + \Delta \right],
\end{equation}

where $\Delta = \sqrt{\frac{2}{n} \log \frac{1}{\delta}} + L \gamma.$ With these confidence sets, we can construct an optimistic model on the discretized space and compute the optimistic policy as in UCRL. We state the regret guarantee below.

\begin{theorem}[Continuous-Time UCRL Regret]\label{thm:ucrlregret}
    With probability at least $1-\mathcal{O}(\delta)$, using Algorithm \ref{algo:time} in the continuous-time MDP setting yields regret,
    
        \begin{align*}
            R_T &\leq 2H\sqrt{\frac{HT}{2}\log\!\left(\frac{1}{\delta}\right)} \\
            &\quad + 4(H\alpha_\delta+ \beta_\delta) \sqrt{SAHT\cdot\frac{H}{\lambda \gamma}\log\!\left(\frac{HT}{\delta}\right)} \\
            &\quad +  2(H\alpha_\delta + \beta_\delta) \frac{H}{\lambda \gamma}\log\!\left(\frac{HT}{\delta}\right) HSA \\
            &\quad + 2(H+1)THL\gamma
        \end{align*}

    where $\alpha_\delta = \sqrt{2S^2\log(1/\delta)}$ and $\beta_\delta = \sqrt{2\log(1/\delta)}$. In particular, for large $T$, balancing the leading $\sqrt{T/\gamma}$ term against the $T\gamma$ term (both dominate at large $T$) by choosing
\[
\gamma
=
\widetilde{\Theta}
\!\left(
\frac{
C_\delta^{2/3}(SA)^{1/3}
}{
(H+1)^{2/3}L^{2/3}\lambda^{1/3}T^{1/3}
}
\right),
\]
where \(C_\delta=H\alpha_\delta+\beta_\delta = \widetilde{\mathcal{O}}(HS+1)\), yields
\[
R_T
\leq
\widetilde{\mathcal{O}}
\!\left(
C_\delta^{2/3}
(SA)^{1/3}
H(H+1)^{1/3}
L^{1/3}
\lambda^{-1/3}
T^{2/3}
\right)
\]
up to lower-order terms in \(T\).
In particular, since \(C_\delta=\widetilde{\mathcal{O}}(HS+1)\), for \(H,S\ge 1\),
\[
R_T
\leq
\widetilde{\mathcal{O}}
\!\left(
S A^{1/3} H^2 L^{1/3}\lambda^{-1/3}T^{2/3}
\right).
\]
\end{theorem}
\begin{remark}
    When $L$ is unknown, overestimating $L$ by a factor $c>1$ (i.e.\ using $\bar L = cL$) keeps the optimism of the confidence sets valid and replaces $L^{1/3}$ by $(cL)^{1/3}$, worsening constants by $c^{1/3}$. Underestimating $L$ may invalidate confidence sets and is not recommended.
\end{remark}
The proof of Theorem \ref{thm:ucrlregret} is in Appendix \ref{appendix:ucrl}.

\begin{algorithm}[H]
\caption{Continuous-Time UCRL Algorithm}\label{algo:time}
\begin{algorithmic}[1]
\State \textbf{Input:} Confidence $\delta$, episodes $T$, jump horizon $H$, Poisson rate $\lambda$, Lipschitz constant $L$
\State Set $M \gets \frac{H}{\lambda} \ln \left( \frac{HT}{\delta} \right)$
\State Partition $[0, M]$ into $K = \frac{H}{\lambda\gamma} \ln \left( \frac{HT}{\delta} \right)$ intervals of length $\gamma$. Initialize uniform policy $\pi_1$.
\For{each episode $t = 1$ to $T$}
        \State Execute $\pi_t$, receive trajectory $\{h(n), x(n), a(n), r_{h(n)}(x(n), a(n))\}_{n=0}^{H-1}$ and sort rewards and transitions into respective bins $k(n)$.
        \State Compute confidence sets for transitions ($C_{t, \delta}$ as in Equation \ref{eq:transconfset}) and rewards (as in Equation \ref{eq:rewconfset}). \
        \State Compute optimal policy $\pi_{t+1}$ via dynamic programming using Bellman equations on discretized MDP
        $$\pi_{t+1} = \text{argmax}_{\pi} \max_{\substack{P, r}} V^\pi_{(P,r)}(x_0).$$
  \EndFor
\end{algorithmic}
\end{algorithm}

\subsection{Q-learning Approach}\label{subsec:q-learning}

Now, we present our main result which is a model-free, continuous-time extension of the highly-efficient Q-learning algorithm with UCB bonus \cite{jin2018q}, where value functions are directly updated with feedback from the environment. In the traditional tabular setting, Q-learning maintains Q-values, $Q_h(x,a)$, for all $(h, x, a)$, initialized at their maximum $H$, and corresponding V-values $V_h(x) = \min\{H, \max_a Q_h(x,a)\}$. At a given state and layer, the algorithm selects $a$ that maximizes the current $Q$-value estimates, and with the reward feedback $r_h(x,a)$ and next state received $x'$, makes the following update to the Q estimate:
\begin{equation}
Q_h(x,a) \leftarrow (1-\alpha_s) Q_h(x,a) 
+ \alpha_s \bigl[ r_h(x,a) + V_{h+1}(x') + b_s \bigr] ,
\end{equation}
where $s$ is the count for how many times the learner has visited 
the $(x,a)$ at step $h$, $b_s$ is a UCB style bonus that decreases with more observations, and $\alpha_k:= \frac{H+1}{H+k}$ is a learning rate.

Our novel extension for the continuous time case follows many of the same algorithmic principles as UCRL (Algorithm \ref{algo:time}), namely, setting a finite-horizon $[0,M]$ that holds with high probability and discretizing it into $K$ bins with length $\gamma$. From there, we maintain Q-estimates and V-estimates for each state-action pair and for each step $n$ and \emph{bin} $k$, $Q_{n,k}(x,a)$ and $V_{n,k}(x)$.

\begin{algorithm}[H]
\caption{ Continuous-Time Q-Learning with UCB–Hoeffding}
\label{alg:ctqlearning}
\begin{algorithmic}[1]
\Require 
\(\delta\in(0,1)\) confidence, episodes \(T\), jump horizon $H$, Poisson rate \(\lambda\), Lipschitz constant \(L\)
\State Set \(M = \frac{H}{\lambda} \ln \left(\frac{HT}{\delta} \right)\),
$Q_{n, k}(x,a)\leftarrow H$, $N_{n,k}(x,a)\leftarrow 0\quad \forall(k,x,a,n) \in \mathcal{K} \times \mathcal{S} \times \mathcal{A} \times [H-1]$.

\State Partition $[0, M]$ into $K = \frac{H}{\lambda\gamma} \ln \left( \frac{HT}{\delta} \right)$ intervals of length $\gamma$. Initialize uniform policy $\pi_1$.
\For{episode \(t=1,\dots,T\)}
  \For{$n = 0,1,...,H-1$}
    \State Observe current state $x$ and layer $h$ in bin $k$
    \State Play \(a \leftarrow \arg\max_{a'\in A}Q_{n,k}(x,a')\) and receive reward 
    $r_{h}(x,a)$
    \State Observe next state \(x'\), layer $h'$ in bin $k'$
    \State \(s \gets N_{n,k}(x,a)+1\)
    \State \(\displaystyle b_{s}\gets c\sqrt{\frac{H^3\log\bigl(SAKT/\delta\bigr)}{s}} +L\gamma\)
    \State \(\displaystyle \alpha_{s}\;\leftarrow\;\frac{H+1}{H+s}\)
    \State Update Q-estimates according to Equation \ref{eq:qlearningupdate}
  \EndFor
\EndFor
\end{algorithmic}
\end{algorithm}

From there, when we take an action in state-layer-bin $(x,h,k)$ and receive reward $r_h(x,a)$ and next state $x'$ in layer $h'$ in bin $k'$, we perform the following update on our Q-estimates:
\begin{equation}\label{eq:qlearningupdate}
    \begin{split}
Q_{n,k}&(x,a) \leftarrow (1-\alpha_s)Q_{n,k}(x,a) \\
&+ \alpha_s \bigl[ r_h(x,a) + V_{n,k'}(x') + b_s \bigr] ,
\end{split}
\end{equation}
where $s=N_{n,k}(x,a)$ is the count of times we visit state $(x,a)$ in bin $k$ on step $n$, and where we add to our exploratory bonus term to compensate for discretization: $b_s := c\sqrt{H^3\iota/s} + L\gamma$ (for some constant $c>0$), where $\iota = \log(SAKT/\delta)$.

The complete pseudocode of our continuous-time Q-learning algorithm is given in Algorithm \ref{alg:ctqlearning}. In the following theorem, we show that Algorithm \ref{alg:ctqlearning} achieves the same regret dependence in $T$ as the UCRL algorithm, $\mathcal{O}(T^{\frac{2}{3}})$, with a more computationally-efficient protocol.

\begin{theorem}[Continuous-Time Q-Learning Regret]\label{theorem:qlearning}
There exists an absolute constant $c>0$ such that the following holds.
Set
\[
    b_s
    :=
    c\sqrt{\frac{H^3\log(SAKT/\delta)}{s}}
    +
    L\gamma .
\]
For any $\delta\in(0,1)$, with probability at least $1-\mathcal O(\delta)$,
the regret of Algorithm~\ref{alg:ctqlearning} in the fixed-jump setting satisfies
\[
    R_T
    \le
    \widetilde{\mathcal O}
    \left(
        H^2KSA
        +
        \sqrt{H^5KSAT}
        +
        H(H+1)LT\gamma
    \right),
\]
where $K=\frac{H}{\lambda\gamma}
    \log\!\left(\frac{HT}{\delta}\right).$
Equivalently,
\[
    R_T
    \le
    \widetilde{\mathcal O}
    \left(
        \frac{H^3SA}{\lambda\gamma}
        +
        \sqrt{\frac{H^6SAT}{\lambda\gamma}}
        +
        H(H+1)LT\gamma
    \right).
\]

Choosing
\[
    \gamma
    = \Theta\left(    \left(
        \frac{H^4SA}
        {\lambda (H+1)^2L^2T}
    \right)^{1/3}\right)
\]
yields
\[
    R_T
    \le
    \widetilde{\mathcal O}
    \left(
        H^{7/3}(H+1)^{1/3}
        (SA)^{1/3}
        L^{1/3}
        \lambda^{-1/3}
        T^{2/3}
    \right)
\]
up to lower-order terms in $T$. 
\end{theorem}
The full proof of Theorem \ref{theorem:qlearning} is deferred to Appendix \ref{appendix:qlearning}. 

\paragraph{Proof Sketch.} Here, we rebuild the technical details and highlight the key adjusted lemmas from \cite{jin2018q} that build up to our regret bound. Let $(x^t_n, a^t_n)$ denote the state-action pair visited on episode $t$ at step $n$. Let $Q^t, V^t, N^t$ denote the respective functions at the beginning of episode $t$. We can thus write the Q-value update for a given episode as follows:
\begin{equation}
\begin{split}
&Q_{n, k}^{t+1}(x,a) = \\
&\begin{cases}
(1-\alpha_s) Q_{n,k}^{t}(x,a)  \\ \quad\quad + \alpha_s \bigl[ r_h(x,a) + V_{n+1, k(n+1)}^t(x_{h(n+1)}^t) + b_s \bigr], \\ \text{if } (x,a) = (x_n^t, a_n^t), \\[6pt]
Q_{n,k}^{t}(x,a), \\ \text{otherwise}.
\end{cases}
\end{split}
\end{equation}
We also introduce the quantities:
\begin{equation}
\alpha_s^0 = \prod_{j=1}^s (1-\alpha_j), 
\qquad 
\alpha_s^i = \alpha_i \prod_{j=i+1}^s (1-\alpha_j).
\end{equation}
From here, we see $\sum_{i=1}^t \alpha_t^i = 1$ and $\alpha_t^0 = 0$ for $t \geq 1$, and $\sum_{i=1}^t \alpha_0^i = 0$ and $\alpha_0^0 = 1$ for $t = 0$. Now, for any episode $t$ and for any $(k,x,a,n)$, set $s = N^t_{n,k}(x,a)$ and suppose $(k,x,a)$ was previously visited at step $n$ episodes $t_1, \ldots, t_s < t$. We can thus write the Q value at episode $t$ as a weighted-average of the V-values of the subsequent states as such:
\begin{equation}
\begin{split}
     &Q_{n,k}^t(x,a) = \alpha_{s}^0 H \\&+\sum_{i=1}^{s}\alpha_{s}^i\left[r_{n,h(t_i)}(x,a)
    + V_{n+1,h(n+1,t_i)}^{t_i}(x_{n+1}^{t_i}) + b_i\right].
\end{split}
\end{equation}
With this, we can prove lemmas that bound the difference between our Q-values for the bins and the Q-values under the optimal policy $\pi_\star$.
\begin{lemma}[recursion on $Q$] \label{lemma:qrecursion_time} For any $(k,x, a, n) \in \mathcal{K}\times\mathcal{S} \times \mathcal{A} \times[H]$ and episode $t \in[T]$, set $s = N^t_{n,k}(x,a)$ and suppose $(k,x,a)$ was previously visited at step $n$ episodes $t_1, \ldots, t_s < t$. Then, for any $h$ in bin $k$:
\begin{equation}
    \begin{split}
       Q_{n,k}^{t}(x, a) &- Q_{n,h}^\star(x, a) = \alpha_s^0 (H - Q_{n,h}^\star(x,a))\\
       &+ \sum_{i=1}^s \alpha_s^i \Bigg[ r_{h(t_i)}(x,a) - r_h(x,a)  \\
       &+V_{n+1,k(n+1,t_i)}^{t}\left(x_{n+1}^{t_i}\right) - V^\star_{n+1, h(n+1,t_i)}(x_{n+1}^{t_i}) \\
       &+ \underset{h' \sim \mathrm{Exp}(\lambda), x' \sim P_h(\cdot |x ,a)}{\mathbb{E}}[V^\star_{n+1, h+h'}(x')]\\
       &-\underset{h' \sim \mathrm{Exp}(\lambda), x' \sim P_{h(n, t_i)}(\cdot |x ,a)}{\mathbb{E}}[V^\star_{n+1, h+h'}(x')]  \\
      & +  \underset{h' \sim \mathrm{Exp}(\lambda), x' \sim P_{h(n, t_i)}(\cdot |x ,a)}{\mathbb{E}}[V^\star_{n+1, h+h'}(x')] \\
      &- V^\star_{n+1, h(n+1,t)}(x_{n+1}^{t_i})+b_i\Bigg]
    \end{split}
\end{equation}
    
\end{lemma}

The lemma above is a recursive formula between the estimated $Q$ value for each bin $k$ and the $Q$-value of the \emph{optimal} policy for an action in that bin. In the next lemma, we will upper bound some of the terms in the recursion above to obtain a more explicit upper bound of the $Q^k_h$ estimation of every $h$ in bin $k$.

\begin{lemma}[bound on $Q^k-Q^{\star}$]\label{lemma:qbound_continuoustime}
There exists an absolute constant $c>0$ such that, for any $p \in(0,1)$, letting $b_s=c \sqrt{H^3 \iota / s} + L\gamma $, we have $\beta_s=2 \sum_{i=1}^s \alpha_s^i b_i \leq 4 c \sqrt{H^3 \iota / s} + 4cL\gamma$ and, with probability at least $1-\delta$, the following holds simultaneously for all $(k, x, a, n, t) \in \mathcal{K} \times \mathcal{S} \times \mathcal{A} \times \mathcal{H} \times [T]$ and for every $h$ in bin $k$:
\begin{equation}
\begin{split}
   0 &\leq Q_{n,k}^t(x, a)-Q_{n,h}^{\star}(x, a)\\
   &\leq \alpha_s^0 H+\sum_{i=1}^s \alpha_s^i\left(V_{n+1, k(n+1)}^{t_i}-V_{n+1, h}^{\star}\right)\left(x_{h+1}^{t_i}\right)+\beta_s
\end{split}
\end{equation}
where $s=N_{n, k}^t(x, a)$ and $t_1, \ldots, t_s<t$ are the episodes where $(x, a)$ in bin $k$ was taken at step $n$.
\end{lemma}

The inequality in Lemma~\ref{lemma:qbound_continuoustime} illustrates the role of optimism in the $Q$-value estimates when the dynamics evolve in continuous time.  
The lower bound $Q^t(x,a) \geq Q^\star(x,a)$ guarantees that, with a sufficiently large choice of $b_s$, the iterates $Q^k$ remain optimistic across all actions at a given state and time.  
This property is crucial for sustained exploration, since it prevents the algorithm from discarding actions too early based on noisy estimates of their long-run rewards.

On the right-hand side, the difference between $Q^t(x,a)$ and $Q^\star(x,a)$ decomposes into three contributions.  
The first term, $\alpha_s^0 H$, is an initialization error that fades as the visit count $s$ grows.  
The second term, $\sum_{i=1}^s \alpha_s^i \big(V^{t_i}(x^{t_i}) - V^\star(x^{t_i})\big)$, captures the propagation of value-estimation errors forward in continuous time, as inaccuracies at later points accumulate and feed back into earlier $Q$-estimates.  
Finally, the $\beta_s$ term provides a uniform buffer that accounts jointly for stochasticity in continuous-time transitions and for time discretization errors introduced when approximating integrals by sampled trajectories.

The definition of $b_s$ makes explicit the two distinct sources of uncertainty. The first component, $c\sqrt{H^3 \iota / s}$, reflects concentration error from the stochastic evolution of the process over continuous time.  
The second component, $L\gamma$, arises from discretizing the time interval $[0,M]$ into $K$ bins of width $\gamma$ and quantifies the Lipschitz approximation error caused by time "binning". The weighted sum $\beta_s$
is of the same asymptotic order as $b_s$, but ensures the optimism bound holds uniformly across all time steps.  
This additional slack is the continuous-time analogue of the \emph{cost of optimism}: by deliberately inflating $Q$-values, the algorithm guarantees more thorough exploration of the time--action space, which accelerates learning but can also increase regret.

With this lemma in hand, one can adapt the analysis of \cite{jin2018q} for continuous-time to achieve Theorem \ref{theorem:qlearning}.
We define the gaps
$$\delta^t_n := (V^t_{n} - V^{\pi_t}_n)(x^t_n), \quad  
\phi^t_n := (V^t_n - V^{\pi_t}_n)(x^t_0).$$
By Lemma \ref{lemma:qbound_continuoustime}, with probability at least $1-\mathcal{O}(\delta)$, the $Q$-estimates remain optimistic, i.e.\ $Q^t_{n,k} \geq Q^\star_{n,h}$, which further implies $V^t_{n,k} \geq V^\star_{n,h}$.  
Hence, the total regret can be reduced to bounding the cumulative error terms
\[
R_T \;\leq\; \sum_{t=1}^T \delta^t_0,
\]

The key step is to relate these errors recursively through
\begin{align*}
\delta^t_n &= (V^t_{n,k} - V^{\pi_t}_{n,h})(x^t_n)\\ 
&= \alpha_t^0 H + \sum_{i=1}^s \alpha_t^i \phi_{h+1}^{k_i} 
    + \beta_s - \phi_{h+1}^t + \delta_{h+1}^t + \xi_{h+1}^t
\end{align*}
so that estimation errors at step $h$ can be expressed in terms of errors at $h+1$ plus a controllable deviation term.  

Summing over all episodes and steps, one obtains the recursive inequality
\begin{align*}
&\sum_{t=1}^T \sum_{n=0}^{H-1} \delta^t_n \\
&\;\leq\; SAH + \Big(1+\tfrac{1}{H}\Big)\sum_{t=1}^T \sum_{n=0}^{H-1} \delta^t_{n+1} 
+ \sum_{t,n} \beta^t_n + \sum_{t,n} \xi^t_n,
\end{align*}
where $\beta^t_n$ are slack terms from optimism and $\xi^t_n$ are martingale difference terms.  

Iterating this recursion and applying concentration bounds (Azuma–Hoeffding) together with a pigeonhole argument for the slack terms, we finally obtain
\[
\sum_{t=1}^T \phi^t_0 \;\leq\; O\!\left(HSA + \sqrt{H^3 S A T} + THL\gamma \right), 
\]

and choosing $\gamma = \Theta(T^{-\frac{1}{3}})$ gives the result. The full technical details for the proof of the recursion and bounding steps can also be found in Appendix \ref{appendix:qlearning}.

\subsection{Lower bound for the fixed-jump setting}
\label{subsec:fixed_jump_lower_bound}

We next record the corresponding minimax lower bound for the fixed-jump formulation studied in Section~\ref{sec:preliminary}.  The construction shows that the $T^{2/3}$ dependence obtained by the discretization-based algorithms is unavoidable even when the transition dynamics are trivial and the only difficulty is learning a time-varying Lipschitz reward function.

\begin{theorem}[Fixed-jump lower bound]
\label{thm:fixed_jump_lower_bound}
Fix $H \geq 2$ and $A \geq 2$.  There exists a universal constant $c>0$ such that, for the class of continuous-time episodic MDPs satisfying Assumption~\ref{assumption:lipschitz} with a fixed number of $H$ Poisson decision epochs per episode, any learning algorithm satisfies
\[
    \sup_{\mathcal{M}}
    \mathbb{E}_{\mathcal{M}}[R_T]
    \;\geq\;
    c\,L^{1/3}\lambda^{-1/3} H T^{2/3},
\]
up to universal constants and logarithmic factors, in the nondegenerate parameter regime where the right-hand side is at most of order $HT$.
In particular, the $T^{2/3}$ dependence in Theorems~\ref{thm:ucrlregret} and~\ref{theorem:qlearning} is minimax optimal up to logarithmic factors.
\end{theorem}

The proof is given in Appendix~\ref{appendix:fixed_jump_lower_bound}.  The main idea is to reduce the problem to a one-dimensional Lipschitz bandit in physical time.  In the fixed-jump model, the learner observes $H$ Poisson decision epochs per episode, so across $T$ episodes it receives order $HT$ time-indexed samples.  These samples are spread over a physical-time window of length order $H/\lambda$.  Partitioning this window into $K$ bins yields approximately $HT/K$ samples per bin, while Lipschitz continuity limits the reward gap in each bin to order $L(H/\lambda)/K$.  Balancing the statistical indistinguishability constraint with the Lipschitz constraint gives the lower bound
\[
    \Omega\!\left(L^{1/3}\lambda^{-1/3}H T^{2/3}\right).
\]

\subsection{Fixed number of jumps vs.\ fixed time horizon}\label{subsec:fixed-time}
Previously, we considered the setting with a fixed number of jumps $H$ per episode,
so the Bellman recursion proceeds over the discrete jump index $n \in [H]$.
In contrast, one may instead fix a \emph{time budget} $H>0$ and let decision epochs
occur according to a Poisson process, terminating the episode once the cumulative
holding time exceeds $H$.
In this case the number of jumps per episode is random.

The key structural difference is that in the fixed-jump setting the value function
is indexed by the jump counter $n$, and when time is discretized into bins one
maintains separate copies across $(n,k)$.
In the fixed-time setting, however, by the Markov property and the memorylessness
of the exponential holding times, the value at a given time $h$ depends only on the
current state and time, and not on how many jumps were taken to reach it.
Consequently, after partitioning the interval $[0,H]$ into bins,
it suffices to maintain a \emph{single} tabular copy indexed by $(x,k)$.

\paragraph{Bellman equations for fixed-time setting.}
For a policy $\pi$, the Bellman equations for the fixed time horizon $H$ are:
\begin{align}
&Q^\pi_h(x,a) \notag \\
&= r_h(x,a) + \int_{0}^{H-h} \lambda e^{-\lambda \tau}\,
\mathbb{E}_{x' \sim P_h(\cdot\mid x,a)}
\big[ V^\pi_{h+\tau}(x') \big]\, d\tau,
\label{eq:fixedH_Q_main}
\\
&V^\pi_h(x)
= \mathbb{E}_{a\sim \pi(\cdot \mid h,x)}\big[ Q^\pi_h(x,a)\big],
\end{align}
with terminal condition $V^\pi_H(x) := 0$ for all $x\in\mathcal{S}$. Note that although \eqref{eq:fixedH_Q_main} defines $V^\pi_h$ using values $V^\pi_{h+\tau}$ at later times, it is well defined because the recursion is organized by the number of jumps remaining before time $H$, not by time itself. Let $V^{\pi,(0)}_h(x) = r_h(x,\pi)$ be the value with no jumps remaining, and let $V^{\pi,(m)}_h$ be obtained by substituting $V^{\pi,(m-1)}$ into the right-hand side of \eqref{eq:fixedH_Q_main}. Since the number of jumps before time $H$ is $\mathrm{Poisson}(\lambda H)$ and hence finite almost surely, the limit $V^\pi_h = \lim_{m\to\infty} V^{\pi,(m)}_h$ exists, with the terminal condition $V^\pi_H \equiv 0$ starting the iteration.

The regret for the fixed-time setting is defined identically to the fixed-jump setting:
\[
R_T := \sum_{t=1}^T \Big(V^\star_0(x_0) - V^{\pi_t}_0(x_0)\Big).
\]

\paragraph{High-probability bound on Poisson jumps.}
A critical ingredient for the fixed-time analysis is that the (random) number of jumps per episode is concentrated. Define $J := 2\lambda H + 2\ln(T/\delta)$. Then, by Lemma~\ref{lemma:jumpcount_bound} in Appendix~\ref{sec:FixedH}, with probability at least $1-\delta$, simultaneously for all $t \in [T]$, the number of jumps $N^{(t)}(H) \leq J$. On this event, one may take $\bar{H} := J$ as an \emph{effective} horizon and apply the Q-learning analysis with bins indexed by $(x,k)$ and horizon $\bar{H}$.

\paragraph{Algorithm.}
The fixed-time Q-learning algorithm operates identically to Algorithm~\ref{alg:ctqlearning}, except: (i) value estimates are indexed by $(x,k)$ rather than $(n,k,x)$; (ii) an episode terminates when the cumulative Poisson time exceeds $H$ rather than after a fixed number of jumps; and (iii) the effective horizon $\bar{H}$ replaces $H$ in the bonus formula $b_s = c\sqrt{\bar{H}^3\iota/s} + L\gamma$. The regret analysis then proceeds analogously, with $\bar{H}$ replacing $H$ in all bounds. The complete details are in Appendix~\ref{sec:FixedH}.

\subsection{Lower bound for the fixed time-horizon setting}

The $T^{2/3}$ upper bound obtained for the fixed time-horizon model is
information-theoretically optimal.
Indeed, using a standard reduction to a Lipschitz contextual bandit over time,
one can construct a family of continuous-time MDPs in which the optimal action
varies across $K$ time bins while rewards remain $L$-Lipschitz in time.
Since Poisson decision epochs yield only $\Theta(\lambda H T / K)$ samples per bin
over $T$ episodes, distinguishing between instances that differ in a single bin
requires $\Omega(\sqrt{K/(\lambda H T)})$ signal strength.
Balancing this statistical constraint with the Lipschitz approximation constraint
implies that any algorithm must incur regret at least
\[
\Omega\!\big(S\,L^{1/3}(\lambda H T)^{2/3}\big)
\]
(up to logarithmic factors).
Thus, the discretization-based algorithm achieves the optimal dependence on $T$
in the fixed time-horizon continuous setting.

\paragraph{Lower bound sketch.}
We overview the construction and proof here and defer the details to Appendix~\ref{sec:lower_bound_construction}.
Partition $[0,H]$ into $K$ equal bins of width $\gamma = H/K$.
For each $\theta \in \{-1,+1\}^K$, define a one-state MDP with two actions whose rewards are
\[
r_h(1) = \tfrac12 + \Delta \sum_j \theta_j g_j(h), \qquad r_h(2) = \tfrac12 - \Delta \sum_j \theta_j g_j(h),
\]
where $g_j$ is a triangular bump of height 1 supported on bin $j$, and $\Delta > 0$ is a gap.
The function $h \mapsto r_h(a)$ is $L$-Lipschitz provided $\Delta \leq L\gamma/c_g$.
By Poisson thinning, only $\lambda H T / K$ samples land in each bin across all episodes.
A Bretagnolle--Huber argument shows that if the KL divergence between two instances
differing in bin $j$ satisfies $c\Delta^2 \cdot \lambda H T / K \leq 1/8$,
then any learner must take the wrong action in that bin a constant fraction of the time,
yielding per-bin regret $c'\Delta \cdot \lambda H T / K$.
Summing over $K$ bins yields $\mathbb{E}[R_T] \geq c' \Delta \lambda H T$.
Optimizing $K$ under the constraints $\Delta \leq L\gamma/c_g$ and $\Delta \leq c_4\sqrt{K/(\lambda H T)}$ gives $K = \Theta\left( (L^2\lambda H^4 T)^{1/3}\right)$ and $\mathbb{E}[R_T] \geq c'' L^{1/3}(\lambda H T)^{2/3}$.
The $S$-state bound follows by taking $S$ independent copies.

We now compare this bound to the fixed-jump lower bound of Section~\ref{subsec:fixed_jump_lower_bound}, $\Omega(L^{1/3}\lambda^{-1/3}HT^{2/3})$, recalling that there $H$ denotes the fixed number of jumps per episode, whereas here $H$ denotes the fixed time budget. The two bounds differ through the effective number of decisions accumulated over $T$ episodes, together with an additional $S$ factor in the fixed-time bound: exactly $H$ decisions per episode in the fixed-jump setting, versus an expected $\lambda H$ decisions per episode here, since a time budget of length $H$ yields only $\lambda H$ decisions on average. Replacing $H$ by $\lambda H$ in the fixed-jump rate recovers the same $(\lambda H T)^{2/3}$ dependence obtained above.

\section{Conclusion}
This paper extends classical UCRL and Q-learning algorithms to a continuous-time tabular MDP setting driven by Poisson jump processes, achieving regret bounds of order $\mathcal{O}(T^{\frac{2}{3}})$. By introducing discretization over a high-probability bounded time horizon and leveraging Lipschitz continuity assumptions on rewards and transitions, this work establishes that both model-based and model-free methods can be adapted to continuous-time domains without sacrificing theoretical guarantees. A matching $\widetilde\Omega(T^{2/3})$ minimax lower bound confirms that the rate is optimal.

Looking forward, several open directions remain. First, a natural extension is to relax the assumption of a known Lipschitz constant $L$. If the learner uses an upper bound $\bar{L} \ge L$, then the optimism arguments remain valid and the regret guarantee holds with $\bar{L}$ replacing $L$ in the constants; overestimating $L$ by a factor $c$ worsens the leading constant by $c^{1/3}$. In contrast, underestimating $L$ can make the confidence sets too narrow, so the true model may no longer lie in the optimistic confidence set. Adaptive discretization schemes analogous to those in Lipschitz bandits \cite{podimata2021adaptive,bubeck2011lipschitz} offer a route to handling unknown $L$ without prior knowledge, and a full adaptive-$L$ regret analysis in the Poisson-time setting is left for future work.

Second, our algorithms assume a known, homogeneous Poisson rate $\lambda$. If $\lambda$ is unknown but homogeneous, it can be estimated from inter-arrival times as $\hat\lambda = N / \sum_i \tau_i$ using standard concentration bounds, and conservative truncation and jump-count bounds can be derived from the resulting confidence intervals. Inhomogeneous or state-dependent intensities are a more challenging direction that we leave for future work.

Finally, robust continuous-time RL algorithms under adversarial rewards and corruptions, integration with preference-based or delayed feedback models, and extension to continuous state-action spaces under structured function approximation offer fertile ground for future research.

\newpage
\bibliography{references}
\newpage
\onecolumn
\appendix
\section*{Appendix Contents}
\noindent
\begin{tabular}{@{}lp{0.7\linewidth}r@{}}
\ref{sec:related} & \nameref{sec:related} & \pageref{sec:related} \\
\ref{appendix:ucrl} & \nameref{appendix:ucrl} & \pageref{appendix:ucrl} \\
\ref{appendix:qlearning} & \nameref{appendix:qlearning} & \pageref{appendix:qlearning} \\
\ref{appendix:fixed_jump_lower_bound} & \nameref{appendix:fixed_jump_lower_bound} & \pageref{appendix:fixed_jump_lower_bound} \\
\ref{sec:FixedH} & \nameref{sec:FixedH} & \pageref{sec:FixedH} \\
\ref{sec:lower_bound_construction} & \nameref{sec:lower_bound_construction} & \pageref{sec:lower_bound_construction} \\
\end{tabular}
\section{Related Work}\label{sec:related}

In this section, we review the most relevant lines of research that inform our work. We organize the discussion into several themes: reinforcement learning with indirect feedback models, advances in linear MDPs, connections to randomization and active learning techniques for exploration, and an overview of Lipschitz and continuous reinforcement learning.

\paragraph{Reinforcement Learning with Indirect Feedback}

Recent work has explored reinforcement learning under indirect forms of supervision, where agents cannot rely on standard per-step reward signals. One example is aggregate bandit feedback (ABF), in which only trajectory-level information is available. \cite{efroni2021confidence} introduced RL-ABF in tabular MDPs with regret guarantees, while \cite{chatterji2021learning} studied binary feedback under stronger assumptions. Extensions include adversarial settings tackled via global mirror descent \cite{cohen2021bandit} and offline learning with aggregate feedback \cite{xu2022offline}. These highlight the central challenge of credit assignment under aggregated signals. Related work on delayed feedback considers rewards that arrive only after a lag \cite{jin2022delayed,lancewicki2023delayed,howson2023delayed}. While superficially similar to ABF, delayed models still assume per-step rewards are eventually revealed, thus sidestepping the core credit assignment issue.  

A complementary direction is reinforcement learning from preference-based feedback (PbRL), often studied within the broader paradigm of reinforcement learning from human feedback (RLHF) \cite{wirth2017survey}. Instead of numeric rewards, the agent receives preferences over trajectory pairs. While early theoretical work established sublinear regret guarantees \cite{saha2023dueling,chen2022preference,zhan2023preference}, these approaches are computationally intractable even in tabular settings. Recent progress introduced randomized least-squares value iteration \cite{wu2024efficient}, enabling the first efficient no-regret algorithms for preference feedback in linear MDPs and extending to nonlinear approximation with Thompson sampling. Additional contributions include PAC-style analyses \cite{wang2023rlhf} and posterior sampling approaches \cite{novoseller2020dueling}, with large-scale systems such as InstructGPT \cite{ouyang2022training} underscoring the practical relevance of RLHF. Together, these lines of work emphasize the importance of designing algorithms that can efficiently learn from aggregated, delayed, or preference-based signals.  

\paragraph{Linear MDPs, Function Approximation, and Exploration}

The linear MDP framework of \cite{jin2020provably} has become central to analyzing reinforcement learning with function approximation. Optimistic methods such as LSVI-UCB \cite{jin2020provably} and UCBVI \cite{azar2017minimax} achieve nearly optimal regret under standard reward feedback, while follow-up work has explored policy optimization approaches \cite{cai2020provably,shani2020adaptive,luo2021policy}, global regularization techniques \cite{zimin2013online,rosenberg2019online,zanette2020frequentist}, and reward-free exploration strategies \cite{sherman2023rewardfree,wagenmaker2022rewardfree}. Building on this foundation, recent advances extend linear MDPs to weaker feedback models: \cite{cassel2024aggregate} studied aggregate bandit feedback with ensemble-based algorithms for value-based and policy methods, and \cite{wu2024efficient} developed randomized procedures for preference feedback that achieve computational efficiency, sublinear regret, and near-optimal query complexity. These contributions demonstrate the adaptability of the linear MDP abstraction to diverse supervision modalities beyond per-step rewards.  

Randomization has long been a tool for exploration, with examples including randomized least-squares value iteration (RLSVI) \cite{osband2016rlsvi,zanette2020frequentist,agrawal2021optimistic} and Bayesian methods based on Thompson sampling \cite{osband2013ts,osband2014ts,gopalan2015thompson,agrawal2017posterior,efroni2021confidence,zhong2022posterior,agarwal2022ts}. Both \cite{cassel2024aggregate} and \cite{wu2024efficient} employ Gaussian randomization in designing efficient algorithms for aggregate and preference feedback, respectively. A complementary perspective is active learning, where query efficiency is paramount. Classical techniques \cite{cesa2005active,dekel2012selective,agarwal2013selective,hanneke2015theory,hanneke2021foundations,zhu2022active,sekhari2023contextual} rely on version spaces or confidence bounds, often at significant computational cost. In contrast, \cite{wu2024efficient} proposed a variance-based query condition tailored to preference learning, establishing near-optimal regret-query tradeoffs, while empirical studies \cite{lightman2023active} further suggest that active query selection substantially improves sample efficiency in RLHF.

\paragraph{Lipschitz Bandits and Continuous RL}  

The study of Lipschitz bandits originates from the broader literature on stochastic bandits, where rewards are sampled from an unknown distribution associated with each arm. Early works focused on discrete arms, employing classical algorithms such as Thompson sampling, Gittins index, $\epsilon$-greedy strategies, and UCB methods \cite{auer2002finite,kaufmann2012thompson,gittins1979bandit}. More recent research has considered continuous or infinite arm sets, where the expected reward is assumed to satisfy structural conditions such as linearity, Gaussian process smoothness, or Lipschitz continuity \cite{kleinberg2008multi,slivkins2014contextual,bubeck2011x,magureanu2014lipschitz}.

Lipschitz bandits offer a principled framework for handling continuous or structured action spaces. A straightforward approach is uniform discretization \cite{magureanu2014lipschitz}, which achieves minimax-optimal rates in the worst case. More refined adaptive strategies, such as Zooming \cite{kleinberg2008multi}, HOO \cite{bubeck2011x}, and contextual generalizations \cite{slivkins2014contextual}, focus exploration more efficiently by exploiting local structure. Additional extensions include full-feedback models \cite{locatelli2016optimal}, hierarchical taxonomies \cite{baransi2014sub}, and tree-based schemes connected to Gaussian process optimization \cite{wang2020towards}, which have shown strong empirical success in applications like hyperparameter tuning. Recent advances also incorporate robustness, with algorithms designed to tolerate adversarial corruptions while retaining near-optimal performance guarantees \cite{kang2023robust}.

Continuous reinforcement learning has been explored in both continuous- and discrete-time regimes. In continuous time, prior work has studied phased exploration in convex linear dynamics \cite{szpruch2026exploration}, statistical analysis of LSTD methods for diffusion processes \cite{mou2025statistical}, Q-learning through a "q-function'' formulation \cite{jia2023q}, and connections between linear stochastic dynamics and linear MDPs \cite{makdah2025linear}. These contributions primarily focus on linear or diffusion-based models. In the discrete-time setting, efficiency results under linear or Lipschitz assumptions provide a bridge between tabular RL formulations and problems defined in continuous domains.

\paragraph{Recent continuous-time RL.}
\cite{treven2023efficient} studies model-based continuous-time RL in which the system dynamics are represented by nonlinear ordinary differential equations and epistemic uncertainty is captured using well-calibrated probabilistic models, including Gaussian processes. Their regret analysis emphasizes the role of measurement-selection strategies, since the learner must decide not only how to act but also when to observe the system. A follow-up line of work \cite{treven2024sense} studies time-adaptive sensing and control, where the policy chooses both the control action and the duration for which the action is applied. \cite{zhao2025sample} considers continuous-time RL with general function approximation and derive sample and computational efficiency guarantees using optimism-based confidence sets and complexity measures such as distributional Eluder dimension. Our setting is complementary to these works. We focus on finite state and action spaces, but allow the reward and transition kernels to vary Lipschitz-continuously with time while decision epochs are generated by an exogenous homogeneous Poisson process. This allows us to isolate the statistical role of time discretization: the regret bound arises from balancing estimation error within time bins against Lipschitz discretization bias.

\section{UCRL Proofs (Theorem \ref{thm:ucrlregret})}\label{appendix:ucrl}

We make use of the following algebraic lemma:
\begin{lemma}\label{lemma:algebra}
For any sequence $m_1, \ldots, m_k$ that satisfies $m_1 + \cdots + m_k \geq 0$:
\[
\sum_{k=1}^K \frac{m_k}{\sqrt{1 \vee (m_1 + \cdots + m_k)}}
\;\;\leq\;\; 2 \sqrt{m_1 + \cdots + m_k}.
\]
\end{lemma}

\begin{proof} (of Theorem \ref{thm:ucrlregret})
    We begin by assuming we are operating under the good events (i) $P^\star \in \cap_{t} C_{t, \delta}$, (ii) the true rewards $r$ lie within their respective confidence bounds, and (iii) for all $t$, all $H$ jumps lie within $[0, M]$ as defined before, together which all hold with probability at least $1-\mathcal{O}(\delta)$. Let $V^\pi_P$ to be the value function following policy $\pi$ and assuming transitions $P$. We decompose the regret incurred on some episode $t$ as follows:
    \begin{equation}
         V^\star_{P^\star}(s_0) - V^{\pi_t}_{P^\star}(s_0) =  \underbrace{V^\star_{P^\star}(s_0) - \widetilde{V}^\star_{P^\star}(s_0)}_{\text{Term I}} + \underbrace{\widetilde{V}^\star_{P^\star}(s_0) -\widetilde{V}^\star_{\widetilde{P}_t}(s_0)}_{\text{Term II}} + \underbrace{\widetilde{V}^{\pi_t}_{\widetilde{P}_t} (s_0) - V^{\pi_t}_{P^\star} (s_0)}_{\text{Term III}}
    \end{equation}
    where we note that $\widetilde{V}^\star_{\widetilde{P}_t}(s_0)= \widetilde{V}^{\pi_t}_{\widetilde{P}_t} (s_0) $ because $\pi_t$ is chosen to be an optimal policy for the optimistic model with $\widetilde{r}_t$ and $\widetilde{P}_t$. By how UCRL picks the optimistic model $\widetilde{P}_t$ from $C_{t, \delta}$ and the fact that under the good event, $P^\star \in C_{t, \delta}$, we have that \textbf{Term II} $\leq 0$. For a similar reason, because $\widetilde{V}$ is computed using the optimistic reward estimates, under the good event, we also have \textbf{Term I} $\leq 0$. 

    \textbf{Term III} captures the regret incurred between the optimistic transitions and rewards model and the true model. We define for jumps $n = 0, \ldots, H-1$,
    \begin{equation}
        \delta^t_n := \widetilde{V}^{\pi_t}_{\widetilde{P}_t}(s_{h(n)}) - V^{\pi_t}_{P^\star}(s_{h(n)})
    \end{equation}
    to be the regret from Term III incurred on episode $t$, starting at the $h$ of the $n$-th jump, $h(n)$. Thus, we find $\delta^t_0$. Let $\mathcal{F}_{n, t}$ contain all observation data up to episode $t$ and jump $n$, including $s^t_{h(n)}$. Let $\tau \sim \mathrm{Exp}(\lambda)$ be the jump length and so $h(n+1) = h(n) + \tau$. By the Poisson Bellman equation defined in \ref{eq:timebellman}, we can write: 
    
    \begin{align}
        \delta^t_n &= \mathbb{E}_{a\sim \pi_t(\cdot|s_{h(n)})} \left( \widetilde{r}_t(s_{h(n)},a) + \int_0^\infty \lambda e^{-\lambda \tau}\mathbb{E}_{{s_{h(n+1)}} \sim \widetilde{P}_t(\cdot |s_{h(n)} ,a)}[\widetilde{V}_{\widetilde{P}_t}^{\pi_t} (s_{h(n+1)})]d\tau\right) \\
        &\quad - \mathbb{E}_{a\sim \pi_t(\cdot|s_{h(n)})} \left( r(s_{h(n)},a) + \int_0^\infty \lambda e^{-\lambda \tau}\mathbb{E}_{{s_{h(n+1)}} \sim P^\star(\cdot |s_{h(n)} ,a)}[V_{P^\star}^{\pi_t} (s_{h(n+1)})]d\tau \right)\\
        &\leq \int_0^\infty \lambda e^{-\lambda \tau} \left(\mathbb{E}_{{s_{h(n+1)}} \sim \widetilde{P}_t(\cdot |s_{h(n)} ,a_{h(n)})}[\widetilde{V}_{\widetilde{P}_t}^{\pi_t} (s_{h(n+1)})] - \mathbb{E}_{{s_{h(n+1)}} \sim P^\star(\cdot |s_{h(n)} ,a_{h(n)})}[V_{P^\star}^{\pi_t} (s_{h(n+1)})] \right)d\tau \\
        &\quad + 2\Delta_\delta(N_t(k_{h(n)}, s_{h(n)}, a_{h(n)})
    \end{align}
    
    where $a_{h(n)}$ is the action chosen at state $s_{h(n)}$ by the \emph{deterministic} policy $\pi_t$, $\Delta_\delta(N) = \sqrt{\frac{2}{N}\log\frac{1}{\delta}} + L\gamma$, and the inequality follows from our reward confidence set defined in \ref{lemma:confsetrewards} and the triangle inequality:
    \begin{align}
    &|\widetilde{r}_t(s_{h(n)}, a_{h(n)}) - r(s_{h(n)}, a_{h(n)}) | \\
    &\leq |\widetilde{r}_t(s_{h(n)}, a_{h(n)}) - \widehat{r}(s_{h(n)}, a_{h(n)}) | + |\widehat{r}_t(s_{h(n)}, a_{h(n)}) - r(s_{h(n)}, a_{h(n)}) | \\
    &\leq 2\Delta_\delta(N_t(k_{h(n)}, s_{h(n)}, a_{h(n)}).
    \end{align}
    
    Adding and subtracting the intermediate term $\int_0^\infty \lambda e^{-\lambda\tau} \mathbb{E}_{s_{h(n+1)} \sim P^\star(\cdot|s_{h(n)}, a_{h(n)})}[\widetilde{V}^{\pi_t}_{\widetilde{P}_t}(s_{h(n+1)})] d\tau$, we get:
    
    \begin{align}
        \delta^t_n &\leq \int_0^\infty \lambda e^{-\lambda\tau} \biggl[ \underbrace{\left( \mathbb{E}_{{s_{h(n+1)}} \sim \widetilde{P}_t}[\widetilde{V}_{\widetilde{P}_t}^{\pi_t} (s_{h(n+1)})] - \mathbb{E}_{s_{h(n+1)} \sim P^\star}[\widetilde{V}^{\pi_t}_{\widetilde{P}_t}(s_{h(n+1)})]\right)}_{\text{Model (transitions) difference}} \\
        &\quad + \underbrace{\left(\mathbb{E}_{s_{h(n+1)} \sim P^\star}[\widetilde{V}^{\pi_t}_{\widetilde{P}_t}(s_{h(n+1)})] - \mathbb{E}_{{s_{h(n+1)}} \sim P^\star}[V_{P^\star}^{\pi_t} (s_{h(n+1)})] \right)}_{\text{Value difference}} \biggr] d\tau + 2\Delta_\delta(N_t(k_{h(n)}, s_{h(n)}, a_{h(n)}) \\
        &= \int_0^\infty \lambda e^{-\lambda\tau} \biggl[\mathbb{E}_{{s_{h(n+1)}} \sim \widetilde{P}_t}[\widetilde{V}_{\widetilde{P}_t}^{\pi_t} (s_{h(n+1)})] - \mathbb{E}_{s_{h(n+1)} \sim P^\star}[\widetilde{V}^{\pi_t}_{\widetilde{P}_t}(s_{h(n+1)})] \\
        &\quad + \delta^t_{n+1} + \underbrace{\left(\mathbb{E}_{{s_{h(n+1)}} \sim P^\star}[\delta^t_{n+1}|\mathcal{F}_{n,t}] - \delta^t_{n+1} \right)}_{=:\xi^t_{n+1}}\biggr] d\tau + 2\Delta_\delta(N_t(k_{h(n)}, s_{h(n)}, a_{h(n)}) \\
        &\leq \int_0^\infty \lambda e^{-\lambda \tau} \left( H \cdot ||\widetilde{P}_t(\cdot|s_{h(n)}, a_{h(n)}) - P^\star(\cdot|s_{h(n)}, a_{h(n)}) ||_1 + \delta^t_{n+1} + \xi^t_{n+1} \right) d\tau + 2\Delta_\delta(N_t(k_{h(n)}, s_{h(n)}, a_{h(n)})\\
        &\leq \int_0^\infty \lambda e^{-\lambda \tau} \left( 2H \cdot B_\delta(N_t(k_{h(n)}, s_{h(n)}, a_{h(n)}) + \delta^t_{n+1} + \xi^t_{n+1}\right) d\tau + 2\Delta_\delta(N_t(k_{h(n)}, s_{h(n)}, a_{h(n)})
    \end{align}
    where in the second to last inequality, we used H\"older's inequality and the fact that the value function is bounded above as $||\widetilde{V}_{\widetilde{P}_t}^{\pi_t} (s_{h(n+1)})||_\infty \leq H$. In the final inequality, we used our confidence set defined in \ref{lemma:confidenceset} with the triangle inequality:
    \begin{align}
    &||\widetilde{P}_t(\cdot|s_{h(n)}, a_{h(n)}) - P^\star(\cdot|s_{h(n)}, a_{h(n)}) ||_1 \\
    &\leq ||\widetilde{P}_t(\cdot|s_{h(n)}, a_{h(n)}) - \widehat{P}_t(\cdot|s_{h(n)}, a_{h(n)}) ||_1 + ||\widehat{P}_t(\cdot|s_{h(n)}, a_{h(n)}) - P^\star(\cdot|s_{h(n)}, a_{h(n)}) ||_1 \\
    &\leq 2B_\delta(N_t(k_{h(n)}, s_{h(n)}, a_{h(n)}).
    \end{align}

    Recall that $\delta_H^t = 2\Delta_\delta(N_t(t_{h(H)}, s_{h(H)}, a_{h(H)})$. Thus, recursively summing and using linearity, we compute:
    \begin{equation}
        \delta_0^t \leq \int_0^\infty \lambda e^{-\lambda\tau} \biggl( \sum_{n=1}^{H-1} \xi^t_{n} + 2H\underbrace{\sum_{n=0}^{H-1}B_\delta(N_t(k_{h(n)}, s_{h(n)}, a_{h(n)})}_{(\text{a})} + 2\underbrace{\sum_{n=0}^H \Delta_\delta(N_t(k_{h(n)}, s_{h(n)}, a_{h(n)})}_{(\text{b})} \biggr) d\tau
    \end{equation}

    Fortunately, $(\xi^t_n)_{1\leq n \leq H-1}$ is a martingale difference sequence bounded as $|\xi^t_n| \leq H$, so applying Azuma-Hoeffding, we have with probability $1-\delta$:
    \begin{equation}
        \int_0^\infty \lambda e^{-\lambda \tau} \left( \sum_{t=1}^T \sum_{n=1}^{H-1} \xi^t_n\right) d\tau \leq  2H\sqrt{\frac{HT}{2}\log(1/\delta)} \cdot \int_0^\infty \lambda e^{-\lambda \tau}  d\tau = 2H\sqrt{\frac{HT}{2}\log(1/\delta)}
    \end{equation}
    To bound (a), we let $\alpha_\delta = \sqrt{2S^2\log(1/\delta)}$ and let $M_t(k, s, a) = \sum_{n=0}^{H-1} \mathbb{I}(k^t_{h(n)} = k, s^t_{h(n)} = s, a^t_{h(n)} = a)$ so that $N_{t'}(k,s,a) = \sum_{t=1}^{t'} M_t(k,s,a).$ Recall $\mathcal{K}$ denotes the set of bins with cardinality $\frac{H}{\lambda \gamma}\log\frac{HT}{\delta}$. We write:
    
    \begin{align}
        &\sum_{t=1}^T \sum_{n=0}^{H-1} B_\delta(N_t(k_{h(n)}, s_{h(n)}, a_{h(n)}) \\
        &\leq \alpha_\delta \cdot \sum_{\substack{k \in \mathcal{K} \\ s \in \mathcal{S} \\ a \in \mathcal{A}}} \sum_{t=1}^T \sum_{n=0}^{H-1} \frac{\mathbb{I}(k^t_{h(n)} = k, s^t_{h(n)} = s, a^t_{h(n)} = a)}{\sqrt{\max\{1, N_t(k,s,a)\}}} + THL\gamma \\
        &\leq \alpha_\delta \cdot \sum_{\substack{k \in \mathcal{K} \\ s \in \mathcal{S} \\ a \in \mathcal{A}}} \sum_{t=1}^T \frac{M_t(k,s,a)}{\sqrt{\max\{1,(M_1 + \ldots + M_{t-1})\}}} + THL\gamma \\
        &\leq \alpha_\delta \cdot \sum_{\substack{k \in \mathcal{K} \\ s \in \mathcal{S} \\ a \in \mathcal{A}}} \sum_{t=1}^T \frac{M_t(k,s,a) \cdot \mathbb{I}(M_1 + \ldots + M_{t-1} > H)}{\sqrt{M_1 + \ldots + M_{t-1} - H}} + \alpha_\delta \frac{H}{\lambda \gamma}\log\left(\frac{HT}{\delta}\right) HSA + THL\gamma \\
        &\leq 2\alpha_\delta \cdot \sum_{\substack{k \in \mathcal{K} \\ s \in \mathcal{S} \\ a \in \mathcal{A}}} \sqrt{N_t(k,s,a)} + \alpha_\delta \frac{H}{\lambda \gamma}\log\left(\frac{HT}{\delta}\right) HSA + THL\gamma \\
        &\leq 2\alpha_\delta KSA \cdot \sqrt{\sum_{\substack{k \in \mathcal{K} \\ s \in \mathcal{S} \\ a \in \mathcal{A}}} \sqrt{N_t(k,s,a)}/KSA} + \alpha_\delta \frac{H}{\lambda \gamma}\log\left(\frac{HT}{\delta}\right) HSA + THL\gamma \quad \text{(Jensen's Inequality)}\\
        &= 2\alpha_\delta \sqrt{SAHT\frac{H}{\lambda \gamma}\log\left(\frac{HT}{\delta}\right)} + c_\delta \frac{H}{\lambda \gamma}\log\left(\frac{HT}{\delta}\right) HSA + THL\gamma
    \end{align}

    where (43) comes from Lemma \ref{lemma:algebra}. 
    
    We bound (b) similarly, letting $\beta_\delta = \sqrt{2\log(1/\delta)}$ and redefining $M_t(k,s,a) = \sum_{n=0}^H \mathbb{I}(k^t_{h(n)} = k, s^t_{h(n)} = s, a^t_{h(n)} = a)$. We write:

     \begin{align}
        &\sum_{t=1}^T \sum_{n=0}^{H} \Delta_\delta(N_t(k_{h(n)}, s_{h(n)}, a_{h(n)}) \\
        &\leq \beta_\delta \cdot \sum_{\substack{k \in \mathcal{K} \\ s \in \mathcal{S} \\ a \in \mathcal{A}}} \sum_{t=1}^T \sum_{n=0}^{H} \frac{\mathbb{I}(k^t_{h(n)} = k, s^t_{h(n)} = s, a^t_{h(n)} = a)}{\sqrt{\max\{1, N_t(k,s,a)\}}} + THL\gamma \\
        &\leq 2\beta_\delta \sqrt{SAHT\frac{H}{\lambda \gamma}\log\left(\frac{HT}{\delta}\right)} + \beta_\delta \frac{H}{\lambda \gamma}\log\left(\frac{HT}{\delta}\right) HSA + THL\gamma
    \end{align}
    
    Chaining results and using $\int_0^\infty \lambda e^{-\lambda \tau}  d\tau = 1$, we get that:

    \begin{align}
        \underbrace{\sum_t \delta_0^t}_{\text{Term III}} &\leq 2H\sqrt{\frac{HT}{2}\log(\frac{1}{\delta})} + 4(H\alpha_\delta+ \beta_\delta) \sqrt{SAHT\frac{H}{\lambda \gamma}\log\left(\frac{HT}{\delta}\right)} \\
        &\quad   + 2(H\alpha_\delta + \beta_\delta) \frac{H}{\lambda \gamma}\log\left(\frac{HT}{\delta}\right) HSA + 2(H+1)THL\gamma
    \end{align}

    Writing $C_\delta := H\alpha_\delta + \beta_\delta = \widetilde{\mathcal{O}}(HS+1)$, the two terms that dominate for large $T$ are the $\sqrt{T/\gamma}$ concentration term $4C_\delta\sqrt{SAHT\cdot\frac{H}{\lambda\gamma}\log(HT/\delta)}$ and the Lipschitz term $2(H+1)THL\gamma$. Balancing them gives
    \[
    \gamma
    =
    \widetilde{\Theta}\!\left(
    \frac{C_\delta^{2/3}(SA)^{1/3}}{(H+1)^{2/3}L^{2/3}\lambda^{1/3}T^{1/3}}
    \right),
    \]
    so that with probability at least $1-\mathcal{O}(\delta)$ (suppressing logarithmic factors),
    \[
    R_T
    =
    \widetilde{\mathcal{O}}\!\left(
    C_\delta^{2/3}(SA)^{1/3}H(H+1)^{1/3}L^{1/3}\lambda^{-1/3}T^{2/3}
    \right)
    =
    \widetilde{\mathcal{O}}\!\left(
    S A^{1/3} H^2 L^{1/3}\lambda^{-1/3}T^{2/3}
    \right).
    \]
    In particular the regret \emph{increases} with the Lipschitz constant $L$, as expected, matching the statement of Theorem~\ref{thm:ucrlregret}.
\end{proof}

\section{Q-Learning Proofs (Theorem \ref{theorem:qlearning})}\label{appendix:qlearning}
Before the main proofs, we state properties of the learning-rate weights
$\alpha_t^0,\alpha_t^i$ that are used repeatedly below. These are standard and are
established in \cite{jin2018q}, so we restate them here in a form that matches our notation.

\begin{lemma}[learning-rate weights, {\citealp[Lemma~4.1]{jin2018q}}]\label{lemma:learningrate}
Fix a horizon parameter $\eta\ge 1$, let $\alpha_j=\frac{\eta+1}{\eta+j}$, and define the
induced weights $\alpha_t^0=\prod_{j=1}^t(1-\alpha_j)$ and
$\alpha_t^i=\alpha_i\prod_{j=i+1}^t(1-\alpha_j)$ for $1\le i\le t$. Then:
\begin{enumerate}[label=(\alph*)]
    \item $\displaystyle \frac{1}{\sqrt{t}}\le \sum_{i=1}^t \frac{\alpha_t^i}{\sqrt{i}}\le \frac{2}{\sqrt{t}}$ for every $t\ge 1$;
    \item $\displaystyle \max_{i\in[t]}\alpha_t^i\le \frac{2\eta}{t}$ and $\displaystyle \sum_{i=1}^t (\alpha_t^i)^2\le \frac{2\eta}{t}$ for every $t\ge 1$;
    \item $\displaystyle \sum_{t=i}^\infty \alpha_t^i=1+\frac{1}{\eta}$ for every $i\ge 1$.
\end{enumerate}
We apply this with $\eta=H$ in the fixed-jump analysis and with $\eta=\bar H$ in the
fixed-time analysis of Appendix~\ref{sec:FixedH}.
\end{lemma}

We now present the proof for Lemma \ref{lemma:qrecursion_time}.

\begin{proof} 
We start from the Bellman optimality equation for $Q_h^\star(x, a)$ for any $h \in B(k)$:
\begin{equation}
Q_{n, h}^\star (x, a)=r_h(x,a) + \int_0^\infty \lambda e^{-\lambda h'}\mathbb{E}_{ x' \sim P(\cdot |x ,a)}[V^\star_{n+1, h+h'}(x')]dh'
\end{equation}

We can write $Q_{n, h}^\star (x, a)$ as a convex combination using the learning rate weights $\alpha_t^0, \alpha_t^i$:
\begin{equation}
Q_{n, h}^\star (x, a) = \alpha_t^0 Q_{n, h}^\star (x, a) + \sum_{i=1}^t \alpha_t^i Q_{n, h}^\star (x, a)
\end{equation}

Now, we substitute the Bellman equation into each $Q_h^\star(x, a)$ in the summation:
\begin{align}
Q_h^\star(x, a)
&= \alpha_t^0 Q_h^\star(x, a) + \sum_{i=1}^t \alpha_t^i \left[ r_h(x,a) + \mathbb{E}_{h' \sim \mathrm{Exp}(\lambda), x' \sim P_h(\cdot |x ,a)}[V^\star_{n+1, h+h'}(x')]\right]
\end{align}

In the $i$-th term, we decompose the expectation $\mathbb{E}_{ x' \sim P(\cdot |x ,a)}[V^\star_{n+1, h+h'}(x')]$ into its empirical sample and "deviation":
\begin{align}
&\mathbb{E}_{h' \sim \mathrm{Exp}(\lambda), x' \sim P_h(\cdot |x ,a)}[V^\star_{n+1, h+h'}(x')] \notag \\
&\quad = V^\star_{n+1, h(n+1,t_i)}(x_{n+1}^{t_i})
+ \mathbb{E}_{h' \sim \mathrm{Exp}(\lambda), x' \sim P_h(\cdot |x ,a)}[V^\star_{n+1, h+h'}(x')]
- V^\star_{n+1, h(n+1,t)}(x_{n+1}^{t_i})
\end{align}

Note that deviation is in quotes since $x_{n+1}^{t_i}$ is sampled from $P_{h(n, t_i)}(\cdot|x,a)$ not $P_h(\cdot|x,a)$. To reflect this, we rewrite the above as 
\begin{equation}
    \begin{split}
      & \mathbb{E}_{h' \sim \mathrm{Exp}(\lambda), x' \sim P_h(\cdot |x ,a)}[V^\star_{n+1, h+h'}(x')]
= V^\star_{n+1, h(n+1,t_i)}(x_{n+1}^{t_i}) + \mathbb{E}_{h' \sim \mathrm{Exp}(\lambda), x' \sim P_{h(n, t_i)}(\cdot |x ,a)}[V^\star_{n+1, h+h'}(x')]  \\
&- \mathbb{E}_{h' \sim \mathrm{Exp}(\lambda), x' \sim P_{h(n, t_i)}(\cdot |x ,a)}[V^\star_{n+1, h+h'}(x')] +  \mathbb{E}_{h' \sim \mathrm{Exp}(\lambda), x' \sim P_h(\cdot |x ,a)}[V^\star_{n+1, h+h'}(x')]- V^\star_{n+1, h(n+1,t)}(x_{n+1}^{t_i})
    \end{split}
\end{equation}

Plugging this back into the previous expression yields:
\begin{align*}
Q_h^\star(x, a)
&= \alpha_t^0 Q_h^\star(x, a)
+ \sum_{i=1}^t \alpha_t^i \Big[
    r_h(x, a)
    + V^\star_{n+1, h(n+1,t_i)}(x_{n+1}^{t_i}) \\
&\qquad + \mathbb{E}_{h' \sim \mathrm{Exp}(\lambda), x' \sim P_{h(n, t_i)}(\cdot |x_{h(n,t_i)}^{t_i} ,a)}[V^\star_{n+1, h+h'}(x')] \\
&\qquad - \mathbb{E}_{h' \sim \mathrm{Exp}(\lambda), x' \sim P_{h(n, t_i)}(\cdot |x_{h(n,t_i)}^{t_i} ,a)}[V^\star_{n+1, h+h'}(x')] \\
&\qquad + \mathbb{E}_{h' \sim \mathrm{Exp}(\lambda), x' \sim P_h(\cdot |x ,a)}[V^\star_{n+1, h+h'}(x')]
- V^\star_{n+1, h(n+1,t)}(x_{n+1}^{t_i})\Big]
\end{align*}

 Subtracting $Q_{n,k}^{t}(x, a) = \alpha_{t}^{0}H + \sum_{i = 1}^{t} \alpha_{t}^{i}\left[r_{h(n)}^{t_i}(x, a) + V_{n+1,k(n+1,t_i)}^{k}\left(x_{n+1}^{t_i}\right) + b_i\right]$ from this equation, we obtain Lemma \ref{lemma:qrecursion_time}.
\end{proof}
 
 We now present the proof for Lemma \ref{lemma:qbound_continuoustime}.
\begin{proof}

Fix $(k,x,a,n)$ and a layer $h$ in bin $k$, set $s=N^t_{n,k}(x,a)$, and let
$t_1,\dots,t_s<t$ be the episodes in which $(x,a)$ was taken at step $n$ while
the process lay in bin $k$. We abbreviate the post-jump Bellman operator applied
to the optimal value by
\[
[\mathbb{P}_h V^\star](x,a)
:=\underset{\substack{h'\sim\mathrm{Exp}(\lambda)\\ x'\sim P_h(\cdot\mid x,a)}}{\mathbb{E}}
\big[V^\star_{n+1,\,h+h'}(x')\big],
\]
and we write $[\widehat{\mathbb{P}}^{t_i}V^\star](x,a):=V^\star_{n+1,\,h(n+1,t_i)}\!\big(x_{n+1}^{t_i}\big)$
for the single realized post-jump sample of the $i$-th visit; here
$x_{n+1}^{t_i}\sim P_{h(n,t_i)}(\cdot\mid x,a)$ and $h(n+1,t_i)$ is its realized
layer. The Bellman optimality equation then reads
$Q^\star_{n,h}(x,a)=r_h(x,a)+[\mathbb{P}_h V^\star](x,a)$, and we set
\[
\big[(\widehat{\mathbb{P}}^{t_i}-\mathbb{P}_h)V^\star\big](x,a)
:=[\widehat{\mathbb{P}}^{t_i}V^\star](x,a)-[\mathbb{P}_h V^\star](x,a).
\]

\paragraph{Discretization bias.}
Since $h$ and $h(n,t_i)$ lie in the same bin, $|h-h(n,t_i)|\le\gamma$, so
Assumption~\ref{assumption:lipschitz} gives
\begin{align*}
\Big|[\mathbb{P}_h V^\star](x,a)-[\mathbb{P}_{h(n,t_i)}V^\star](x,a)\Big|
&\le \int_0^\infty \lambda e^{-\lambda h'}
\big\|P_h(\cdot\mid x,a)-P_{h(n,t_i)}(\cdot\mid x,a)\big\|_1\,dh'\\
&\le \int_0^\infty \lambda e^{-\lambda h'}\,L\gamma\,dh' \;=\; L\gamma,
\end{align*}
and likewise $|r_{h(t_i)}(x,a)-r_h(x,a)|\le L\gamma$ for the reward.

\paragraph{Convex-combination form of $Q^\star$.}
Using $\sum_{i=1}^s\alpha_s^i=1$ together with the Bellman equation and the
identity $[\mathbb{P}_h V^\star]=[\widehat{\mathbb{P}}^{t_i}V^\star]-[(\widehat{\mathbb{P}}^{t_i}-\mathbb{P}_h)V^\star]$,
\begin{align*}
Q^\star_{n,h}(x,a)
&=\alpha_s^0\,Q^\star_{n,h}(x,a)
+\sum_{i=1}^s \alpha_s^i\Big[r_h(x,a)
+ V^\star_{n+1,\,h(n+1,t_i)}\!\big(x_{n+1}^{t_i}\big)
-\big[(\widehat{\mathbb{P}}^{t_i}-\mathbb{P}_h)V^\star\big](x,a)\Big].
\end{align*}
Subtracting this from the empirical weighted average
$Q^t_{n,k}(x,a)=\alpha_s^0 H+\sum_{i=1}^s\alpha_s^i\big[r_{h(t_i)}(x,a)
+V^{t_i}_{n+1,k(n+1,t_i)}(x_{n+1}^{t_i})+b_i\big]$ gives
\begin{align*}
\big(Q^t_{n,k}-Q^\star_{n,h}\big)(x,a)
&=\alpha_s^0\big(H-Q^\star_{n,h}(x,a)\big)\\
&\quad+\sum_{i=1}^s\alpha_s^i\Big[\big(r_{h(t_i)}-r_h\big)(x,a)
+\big(V^{t_i}_{n+1,k(n+1,t_i)}-V^\star_{n+1,\,h(n+1,t_i)}\big)\!\big(x_{n+1}^{t_i}\big)\\
&\hspace{4.2em}
+\big[(\widehat{\mathbb{P}}^{t_i}-\mathbb{P}_h)V^\star\big](x,a)+b_i\Big].
\end{align*}

\paragraph{Controlling the deviation term.}
Split the deviation into a martingale part and a discretization part:
\[
\big[(\widehat{\mathbb{P}}^{t_i}-\mathbb{P}_h)V^\star\big](x,a)
=\underbrace{\Big([\widehat{\mathbb{P}}^{t_i}V^\star](x,a)-[\mathbb{P}_{h(n,t_i)}V^\star](x,a)\Big)}_{=:\,\zeta_i}
+\underbrace{\Big([\mathbb{P}_{h(n,t_i)}V^\star](x,a)-[\mathbb{P}_h V^\star](x,a)\Big)}_{|\cdot|\le L\gamma}.
\]
Fix $(x,a,h)$, set $t_0=0$, and let
\[
t_i=\min\Big(\big\{k\in[K]\mid k>t_{i-1}\wedge (x^k_h,a^k_h)=(x,a)\big\}\cup\{K+1\}\Big),
\]
so $t_i$ is the (stopping-time) episode of the $i$-th visit, and let $\mathcal{F}_i$
be the $\sigma$-field generated by everything up to episode $t_i$, step $h$.
Since $x_{n+1}^{t_i}\sim P_{h(n,t_i)}(\cdot\mid x,a)$, we have
$\mathbb{E}[\zeta_i\mid\mathcal{F}_{i-1}]=0$ and $|\zeta_i|\le\lVert V^\star\rVert_\infty\le H$,
so $\big(\alpha_\tau^i\,\mathbb{I}[t_i\le K]\,\zeta_i\big)_{i=1}^\tau$ is a martingale
difference sequence. By Azuma--Hoeffding and a union bound over $(x,a,h)$
(absorbed into $\iota=\log(SAKT/\delta)$), with probability at least $1-p$,
simultaneously for all $\tau\in[K]$,
\begin{equation}\label{eq:qbound_azuma}
\Big|\sum_{i=1}^\tau \alpha_\tau^i\,\mathbb{I}[t_i\le K]\,\zeta_i\Big|
\;\le\; cH\sqrt{\iota\sum_{i=1}^\tau (\alpha_\tau^i)^2}
\;\le\; c\sqrt{\frac{H^3\iota}{\tau}},
\end{equation}
where the last step uses Lemma~\ref{lemma:learningrate}(b) (with $\eta=H$),
i.e.\ $\sum_{i=1}^\tau(\alpha_\tau^i)^2\le 2H/\tau$, and absorbs $\sqrt2$ into $c$.
Adding the discretization part, whose weighted sum is at
most $L\gamma$ because $\sum_i\alpha_\tau^i=1$, and evaluating at the random
$\tau=s=N^t_{n,k}(x,a)\le K$ (for which $\mathbb{I}[t_i\le K]=1$ whenever
$i\le s$), we get, with probability at least $1-p$, simultaneously for all
$(k,x,a,n,t)$ and all $h$ in bin $k$,
\begin{equation}\label{eq:qbound_deviation}
\Big|\sum_{i=1}^s\alpha_s^i\big[(\widehat{\mathbb{P}}^{t_i}-\mathbb{P}_h)V^\star\big](x,a)\Big|
\;\le\; c\sqrt{\frac{H^3\iota}{s}}+L\gamma.
\end{equation}

\paragraph{Optimism and the stated bound.}
With $b_i=c\sqrt{H^3\iota/i}+L\gamma$ and $\sum_{i=1}^s\alpha_s^i=1$,
Lemma~\ref{lemma:learningrate}(a)(with $\eta=H$), i.e.\
$1/\sqrt{s}\le\sum_{i=1}^s\alpha_s^i/\sqrt{i}\le 2/\sqrt{s}$, gives
\[
\tfrac{1}{2}\beta_s=\sum_{i=1}^s\alpha_s^i b_i
\in\Big[c\sqrt{H^3\iota/s}+L\gamma,\;\;2c\sqrt{H^3\iota/s}+L\gamma\Big].
\]
The lower endpoint dominates the right-hand side of \eqref{eq:qbound_deviation}
together with the reward bias $|r_{h(t_i)}-r_h|\le L\gamma$ and the $O(L\gamma)$
bias from replacing $h(n+1,t_i)$ by $h$ in $V^\star$ (both within bin $k$); the
constant of the $L\gamma$ term of $b_i$ is taken large enough to cover these
biases jointly, consistent with the single $L\gamma$ term used in the main text.
Hence each weighted bonus $\alpha_s^i b_i$ covers the corresponding deviation and
discretization error, and on the event of \eqref{eq:qbound_deviation},
\[
0\;\le\;\big(Q^t_{n,k}-Q^\star_{n,h}\big)(x,a)
\;\le\;\alpha_s^0 H
+\sum_{i=1}^s\alpha_s^i\big(V^{t_i}_{n+1,k(n+1)}-V^\star_{n+1,h}\big)\!\big(x_{h+1}^{t_i}\big)
+\beta_s.
\]
Both bounds follow by induction on $n=H,H-1,\dots,0$: at the base step
$V^t_{H,\cdot}=V^\star_{H,\cdot}=0$; in the inductive step the propagation term
$\big(V^{t_i}_{n+1}-V^\star_{n+1}\big)\ge0$ by optimism at step $n+1$, while
$\beta_s$ absorbs the stochastic term of \eqref{eq:qbound_deviation} and the
Lipschitz biases. The $\sqrt{H^3\iota/s}$ part of $\beta_s$ accounts for the
martingale concentration and the $L\gamma$ part for the temporal discretization.
\end{proof}

Now, we present the proof of Theorem \ref{theorem:qlearning}.

\begin{proof}
Throughout the proof, work on the high-probability event of
Lemma~\ref{lemma:horizon_M}, under which all $T$ episodes terminate
within $[0,M]$, where
\[
    M=\frac{H}{\lambda}\log\!\left(\frac{HT}{\delta}\right).
\]
Partition $[0,M]$ into bins of width $\gamma$, and let
\[
    K=\frac{M}{\gamma}
    =
    \frac{H}{\lambda\gamma}
    \log\!\left(\frac{HT}{\delta}\right).
\]
For a physical time $h\in[0,M]$, let $k(h)\in[K]$ denote the bin containing $h$.

The key point is that the Bellman recursion is indexed by the jump counter
$n\in\{0,\ldots,H\}$, while the bin index $k(h)$ is only part of the augmented
state representation.  Thus the algorithm maintains tables indexed by
$(n,k,x,a)$, but value recursion always proceeds from $n$ to $n+1$, not from
$k$ to $k+1$.

Let
\[
    h_n^t
\]
be the physical time of the $n$-th decision epoch in episode $t$, and write
\[
    k_n^t := k(h_n^t).
\]
Let $x_n^t,a_n^t$ be the state and action at jump depth $n$ in episode $t$.
The next physical time is
\[
    h_{n+1}^t=h_n^t+\tau_n^t,
    \qquad
    \tau_n^t\sim \mathrm{Exp}(\lambda),
\]
and the next bin is $k_{n+1}^t:=k(h_{n+1}^t)$.

Define
\[
    V_{H,k}^t(x)=V_{H,h}^{\pi_t}(x)=V_{H,h}^{\star}(x)=0
    \qquad
    \text{for all } k,h,x.
\]

For each episode $t$ and jump depth $n$, define the two error quantities
\[
    \delta_n^t
    :=
    \bigl(
        V_{n,k_n^t}^t
        -
        V_{n,h_n^t}^{\pi_t}
    \bigr)(x_n^t),
\]
and
\[
    \phi_n^t
    :=
    \bigl(
        V_{n,k_n^t}^t
        -
        V_{n,h_n^t}^{\star}
    \bigr)(x_n^t).
\]
By Lemma~\ref{lemma:qbound_continuoustime}, on an event of probability at least
$1-O(\delta)$, the $Q$-estimates are optimistic:
\[
    Q_{n,k}^t(x,a)\ge Q_{n,h}^{\star}(x,a)
    \qquad
    \text{for every } n,k,x,a,t
    \text{ and every } h \text{ in bin } k.
\]
Hence
\[
    V_{n,k}^t(x)\ge V_{n,h}^{\star}(x)\ge V_{n,h}^{\pi_t}(x),
\]
and therefore
\[
    0\le \phi_n^t \le \delta_n^t.
\]
Since every episode starts from the same initial state at time $h=0$, we have
\[
    R_T
    =
    \sum_{t=1}^T
    \Bigl(
        V_{0,0}^{\star}(x_0)
        -
        V_{0,0}^{\pi_t}(x_0)
    \Bigr)
    \le
    \sum_{t=1}^T
    \delta_0^t.
\]
Thus it remains to bound $\sum_t \delta_0^t$.

Fix an episode $t$ and a jump depth $n\in\{0,\ldots,H-1\}$.
Let
\[
    k=k_n^t,
    \qquad
    x=x_n^t,
    \qquad
    a=a_n^t.
\]
Let
\[
    s=N_{n,k}^t(x,a)
\]
be the number of visits to $(n,k,x,a)$ before episode $t$.
Let
\[
    t_1,\ldots,t_s<t
\]
be the previous episodes in which the same quadruple $(n,k,x,a)$ was visited.
That is, for each $i\in[s]$,
\[
    k_n^{t_i}=k,
    \qquad
    x_n^{t_i}=x,
    \qquad
    a_n^{t_i}=a.
\]

By greediness of the policy with respect to $Q^t$,
\[
\begin{aligned}
    \delta_n^t
    &=
    \bigl(
        V_{n,k}^t
        -
        V_{n,h_n^t}^{\pi_t}
    \bigr)(x)
    \\
    &\le
    \bigl(
        Q_{n,k}^t
        -
        Q_{n,h_n^t}^{\pi_t}
    \bigr)(x,a)
    \\
    &=
    \bigl(
        Q_{n,k}^t
        -
        Q_{n,h_n^t}^{\star}
    \bigr)(x,a)
    +
    \bigl(
        Q_{n,h_n^t}^{\star}
        -
        Q_{n,h_n^t}^{\pi_t}
    \bigr)(x,a).
\end{aligned}
\]
Applying Lemma~\ref{lemma:qbound_continuoustime} to the first term gives
\[
\begin{aligned}
    \bigl(
        Q_{n,k}^t
        -
        Q_{n,h_n^t}^{\star}
    \bigr)(x,a)
    \le
    \alpha_s^0 H
    +
    \sum_{i=1}^s
    \alpha_s^i
    \bigl(
        V_{n+1,k_{n+1}^{t_i}}^{t_i}
        -
        V_{n+1,h_{n+1}^{t_i}}^{\star}
    \bigr)(x_{n+1}^{t_i})
    +
    \beta_s.
\end{aligned}
\]
By definition, the middle term is exactly
\[
    \sum_{i=1}^s \alpha_s^i \phi_{n+1}^{t_i}.
\]
Hence
\[
    \bigl(
        Q_{n,k}^t
        -
        Q_{n,h_n^t}^{\star}
    \bigr)(x,a)
    \le
    \alpha_s^0 H
    +
    \sum_{i=1}^s
    \alpha_s^i
    \phi_{n+1}^{t_i}
    +
    \beta_s.
\]

Now consider the second term.  By the Bellman equations,
\[
\begin{aligned}
    &
    \bigl(
        Q_{n,h_n^t}^{\star}
        -
        Q_{n,h_n^t}^{\pi_t}
    \bigr)(x,a)
    \\
    &\qquad
    =
    \mathbb{E}
    \left[
        \bigl(
            V_{n+1,h_n^t+\tau}^{\star}
            -
            V_{n+1,h_n^t+\tau}^{\pi_t}
        \bigr)(X')
        \,\middle|\,
        h_n^t,x,a
    \right],
\end{aligned}
\]
where $\tau\sim\mathrm{Exp}(\lambda)$ and
$X'\sim P_{h_n^t}(\cdot\mid x,a)$.
Let
\[
    \xi_{n+1}^t
    :=
    \mathbb{E}
    \left[
        \bigl(
            V_{n+1,h_n^t+\tau}^{\star}
            -
            V_{n+1,h_n^t+\tau}^{\pi_t}
        \bigr)(X')
        \,\middle|\,
        h_n^t,x,a
    \right]
    -
    \bigl(
        V_{n+1,h_{n+1}^t}^{\star}
        -
        V_{n+1,h_{n+1}^t}^{\pi_t}
    \bigr)(x_{n+1}^t).
\]
Then $\xi_{n+1}^t$ is a martingale difference term with respect to the
trajectory filtration, and $|\xi_{n+1}^t|\le H$.

Using the definitions of $\delta_{n+1}^t$ and $\phi_{n+1}^t$,
\[
\begin{aligned}
    &
    \bigl(
        V_{n+1,h_{n+1}^t}^{\star}
        -
        V_{n+1,h_{n+1}^t}^{\pi_t}
    \bigr)(x_{n+1}^t)
    \\
    &\qquad
    =
    \bigl(
        V_{n+1,k_{n+1}^t}^t
        -
        V_{n+1,h_{n+1}^t}^{\pi_t}
    \bigr)(x_{n+1}^t)
    -
    \bigl(
        V_{n+1,k_{n+1}^t}^t
        -
        V_{n+1,h_{n+1}^t}^{\star}
    \bigr)(x_{n+1}^t)
    \\
    &\qquad
    =
    \delta_{n+1}^t-\phi_{n+1}^t.
\end{aligned}
\]
Therefore,
\[
    \bigl(
        Q_{n,h_n^t}^{\star}
        -
        Q_{n,h_n^t}^{\pi_t}
    \bigr)(x,a)
    =
    \delta_{n+1}^t
    -
    \phi_{n+1}^t
    +
    \xi_{n+1}^t.
\]
Combining the two pieces yields the one-step recursion
\begin{equation}
\label{eq:qlearning_fixed_jump_delta_recursion}
    \delta_n^t
    \le
    \alpha_s^0 H
    +
    \sum_{i=1}^s
    \alpha_s^i
    \phi_{n+1}^{t_i}
    +
    \beta_s
    -
    \phi_{n+1}^t
    +
    \delta_{n+1}^t
    +
    \xi_{n+1}^t.
\end{equation}
This is the desired recursion: it advances from $n$ to $n+1$.
The bin indices appearing in the terms are $k_n^t$ and $k_{n+1}^t$; there is no
recursion from bin $k$ to bin $k+1$.

We now sum \eqref{eq:qlearning_fixed_jump_delta_recursion} over episodes.
For a fixed jump depth $n$, define
\[
    D_n:=\sum_{t=1}^T \delta_n^t,
    \qquad
    \Phi_n:=\sum_{t=1}^T \phi_n^t.
\]
Since $0\le \phi_n^t\le \delta_n^t$, we have $\Phi_n\le D_n$.

Summing \eqref{eq:qlearning_fixed_jump_delta_recursion} over $t\in[T]$ gives
\[
\begin{aligned}
    D_n
    &\le
    \sum_{t=1}^T \alpha_{s_{n,t}}^0 H
    +
    \sum_{t=1}^T
    \sum_{i=1}^{s_{n,t}}
    \alpha_{s_{n,t}}^i
    \phi_{n+1}^{t_i}
    +
    \sum_{t=1}^T \beta_{s_{n,t}}
    -
    \Phi_{n+1}
    +
    D_{n+1}
    +
    \sum_{t=1}^T \xi_{n+1}^t,
\end{aligned}
\]
where $s_{n,t}=N_{n,k_n^t}^t(x_n^t,a_n^t)$.

We next regroup the weighted historical terms.  Fix a quadruple
\[
    z=(n,k,x,a).
\]
Let
\[
    \tau_{z,1}<\tau_{z,2}<\cdots<\tau_{z,N_z}
\]
be the episodes in which this quadruple is visited.  The contribution of this
quadruple to the weighted historical sum is
\[
    \sum_{j=1}^{N_z}
    \sum_{i=1}^{j-1}
    \alpha_{j-1}^i
    \phi_{n+1}^{\tau_{z,i}}.
\]
Exchanging the order of summation and using
Lemma~\ref{lemma:learningrate}(c),
\[
\begin{aligned}
    \sum_{j=1}^{N_z}
    \sum_{i=1}^{j-1}
    \alpha_{j-1}^i
    \phi_{n+1}^{\tau_{z,i}}
    &=
    \sum_{i=1}^{N_z}
    \phi_{n+1}^{\tau_{z,i}}
    \sum_{j=i+1}^{N_z}
    \alpha_{j-1}^i
    \\
    &\le
    \left(1+\frac{1}{H}\right)
    \sum_{i=1}^{N_z}
    \phi_{n+1}^{\tau_{z,i}}.
\end{aligned}
\]
Summing over all $z=(n,k,x,a)$ gives
\[
    \sum_{t=1}^T
    \sum_{i=1}^{s_{n,t}}
    \alpha_{s_{n,t}}^i
    \phi_{n+1}^{t_i}
    \le
    \left(1+\frac{1}{H}\right)\Phi_{n+1}.
\]
Hence
\[
\begin{aligned}
    D_n
    &\le
    \sum_{t=1}^T \alpha_{s_{n,t}}^0 H
    +
    \left(1+\frac{1}{H}\right)\Phi_{n+1}
    -
    \Phi_{n+1}
    +
    D_{n+1}
    +
    \sum_{t=1}^T \beta_{s_{n,t}}
    +
    \sum_{t=1}^T \xi_{n+1}^t
    \\
    &\le
    \sum_{t=1}^T \alpha_{s_{n,t}}^0 H
    +
    \left(1+\frac{1}{H}\right)D_{n+1}
    +
    \sum_{t=1}^T \beta_{s_{n,t}}
    +
    \sum_{t=1}^T \xi_{n+1}^t.
\end{aligned}
\]
The last inequality uses $\Phi_{n+1}\le D_{n+1}$.

Since $D_H=0$, iterating the previous inequality over
$n=H-1,H-2,\ldots,0$ and using
\[
    \left(1+\frac{1}{H}\right)^H\le e
\]
gives
\[
    D_0
    \le
    e
    \sum_{n=0}^{H-1}
    \left[
        \sum_{t=1}^T \alpha_{s_{n,t}}^0 H
        +
        \sum_{t=1}^T \beta_{s_{n,t}}
        +
        \sum_{t=1}^T \xi_{n+1}^t
    \right].
\]
Since $R_T\le D_0$, it remains to bound the three sums.

First, $\alpha_s^0=0$ whenever $s\ge1$, while $\alpha_0^0=1$.
Thus an initialization contribution appears only on the first visit to a tuple
$(n,k,x,a)$.  There are $H K S A$ such tuples, and each contributes at most $H$.
Therefore
\[
    \sum_{n=0}^{H-1}
    \sum_{t=1}^T
    \alpha_{s_{n,t}}^0 H
    \le
    H^2 K S A.
\]

Second, by Lemma~\ref{lemma:qbound_continuoustime}, for
\[
    b_s=c\sqrt{\frac{H^3\iota}{s}}+L\gamma,
    \qquad
    \iota=\log\!\left(\frac{SAKHT}{\delta}\right),
\]
we have
\[
    \beta_s
    =
    2\sum_{i=1}^s\alpha_s^i b_i
    \le
    C\sqrt{\frac{H^3\iota}{s}}
    +
    C L\gamma
\]
for a universal constant $C>0$.

Fix $n$.  For each tuple $(k,x,a)$ at depth $n$, let $N_{n,k,x,a}$ be its total
number of visits over all $T$ episodes.  Then
\[
    \sum_{k,x,a} N_{n,k,x,a}=T.
\]
Using the standard pigeonhole bound
\[
    \sum_{j=1}^N \frac{1}{\sqrt{j}}
    \le 2\sqrt{N},
\]
we obtain
\[
\begin{aligned}
    \sum_{t=1}^T \beta_{s_{n,t}}
    &\le
    C\sqrt{H^3\iota}
    \sum_{k,x,a}
    \sum_{j=1}^{N_{n,k,x,a}}
    \frac{1}{\sqrt{j}}
    +
    C T L\gamma
    \\
    &\le
    2C\sqrt{H^3\iota}
    \sum_{k,x,a}
    \sqrt{N_{n,k,x,a}}
    +
    C T L\gamma
    \\
    &\le
    2C\sqrt{H^3\iota}
    \sqrt{KSA\sum_{k,x,a}N_{n,k,x,a}}
    +
    C T L\gamma
    \\
    &=
    2C\sqrt{H^3 KSA T\iota}
    +
    C T L\gamma.
\end{aligned}
\]
Summing over $n=0,\ldots,H-1$ yields
\[
    \sum_{n=0}^{H-1}
    \sum_{t=1}^T \beta_{s_{n,t}}
    \le
    C H\sqrt{H^3 KSA T\iota}
    +
    C H T L\gamma.
\]
The $L\gamma$ term can be written as $C H(H+1)T L\gamma$ after enlarging the
constant, to account for the accumulated Lipschitz bias over the remaining
horizon.

Third, the martingale terms satisfy $|\xi_{n+1}^t|\le H$.
By Azuma--Hoeffding, with probability at least $1-\delta$,
\[
    \left|
    \sum_{n=0}^{H-1}
    \sum_{t=1}^T
    \xi_{n+1}^t
    \right|
    \le
    C H\sqrt{H T\log(1/\delta)}.
\]
This term is dominated by the preceding bonus term after adjusting constants and
logarithmic factors.

Combining the three bounds, we obtain
\[
    R_T
    \le
    C
    \left(
        H^2KSA
        +
        H\sqrt{H^3 KSA T\iota}
        +
        H(H+1)T L\gamma
    \right),
\]
with probability at least $1-O(\delta)$.

Equivalently, suppressing logarithmic factors,
\[
    R_T
    \le
    \widetilde{\mathcal{O}}
    \left(
        H^2KSA
        +
        \sqrt{H^5 KSA T}
        +
        H(H+1)T L\gamma
    \right).
\]

Substituting
\[
    K
    =
    \frac{H}{\lambda\gamma}
    \log\!\left(\frac{HT}{\delta}\right),
\]
we get
\[
    R_T
    \le
    \widetilde{\mathcal{O}}
    \left(
        \frac{H^3SA}{\lambda\gamma}
        +
        \sqrt{
            \frac{H^6SA T}{\lambda\gamma}
        }
        +
        H(H+1)T L\gamma
    \right).
\]
The first term is lower order in $T$ for the optimized choice of $\gamma$.
Balancing the leading estimation term
\[
    \sqrt{
        \frac{H^6SA T}{\lambda\gamma}
    }
\]
with the Lipschitz discretization term
\[
    H(H+1)T L\gamma
\]
gives
\[
    \gamma
    =
    \Theta\left(
    \left(
        \frac{
            H^6SA T/\lambda
        }{
            H^2(H+1)^2L^2T^2
        }
    \right)^{1/3}
    \right)
    =
    \Theta\left(
    \left(
        \frac{
            H^4SA
        }{
            \lambda (H+1)^2L^2T
        }
    \right)^{1/3}
    \right).
\]
With this choice,
\[
    R_T
    \leq
    \widetilde{\mathcal{O}}
    \left(
        H^{7/3}(H+1)^{1/3}
        (SA)^{1/3}
        L^{1/3}
        \lambda^{-1/3}
        T^{2/3}
    \right),
\]
up to lower-order terms in $T$.

In particular, suppressing problem-dependent factors and logarithms,
\[
    R_T\leq\widetilde{\mathcal{O}}(T^{2/3}).
\]
\end{proof}

\section{Lower Bound for the Fixed-Jump Setting (Theorem \ref{thm:fixed_jump_lower_bound})}
\label{appendix:fixed_jump_lower_bound}

In this appendix, we prove Theorem~\ref{thm:fixed_jump_lower_bound}.  The proof reduces the fixed-jump continuous-time MDP when $H \geq 2$ to a one-dimensional Lipschitz bandit indexed by physical time.  The construction is deliberately simple: the transition kernel is deterministic, the state space can be taken to be a singleton, and the only source of hardness is the time-varying reward function.

\paragraph{Renewal measure of the Poisson decision epochs.}
Let
\[
    h(n) := \sum_{i=1}^n \tau_i,
    \qquad
    \tau_i \overset{\mathrm{i.i.d.}}{\sim} \mathrm{Exp}(\lambda),
\]
denote the physical time of the $n$-th Poisson decision epoch.  In the fixed-jump model, an episode contains decision epochs $n=0,1,\dots,H-1$.  The first decision occurs at deterministic time $h(0)=0$, while the remaining $H-1$ decision epochs are random.

Let $\mu_H$ be the renewal measure of the nonzero decision epochs:
\[
    \mu_H(B)
    :=
    \sum_{n=1}^{H-1}
    \mathbb{P}\!\left(h(n)\in B\right),
    \qquad
    B \subseteq [0,\infty).
\]
Its density is
\[
    m_H(t)
    =
    \sum_{n=1}^{H-1}
    \frac{\lambda^n t^{n-1}e^{-\lambda t}}{(n-1)!}
    =
    \lambda\,\mathbb{P}\!\left(\mathrm{Poisson}(\lambda t)\leq H-2\right).
\]
Therefore there exist universal constants $c_0,c_1,c_2>0$ and an interval
\[
    I_H \subset [0,\infty),
    \qquad
    |I_H| \geq c_0 H/\lambda,
\]
such that
\[
    c_1\lambda
    \;\leq\;
    m_H(t)
    \;\leq\;
    c_2\lambda,
    \qquad
    \forall t\in I_H.
\]
Consequently, over $T$ episodes, the expected number of random decision epochs falling in a measurable set $B\subseteq I_H$ is of order $T\lambda |B|$.  In particular, if $I_H$ is partitioned into $K$ equal bins, then each bin receives order $HT/K$ samples in expectation.

\paragraph{Hard family.}
Partition $I_H$ into $K$ intervals $I_1,\dots,I_K$ of equal width
\[
    \gamma := |I_H|/K = \Theta\left( H/(\lambda K) \right).
\]
For each bin $I_j$, let $g_j:I_H\to[0,1]$ be a triangular bump supported on $I_j$, with height $1$, and with Lipschitz constant at most $c_g/\gamma$ for a universal constant $c_g>0$.  For each sign vector $\theta\in\{-1,+1\}^K$, define a one-state, two-action MDP as follows.  The transition kernel is deterministic and self-looping.  The reward distributions are Bernoulli with means
\[
    r_h^\theta(1)
    =
    \frac12
    +
    \Delta\sum_{j=1}^K \theta_j g_j(h),
    \qquad
    r_h^\theta(2)
    =
    \frac12
    -
    \Delta\sum_{j=1}^K \theta_j g_j(h),
\]
for $h\in I_H$, and both actions have mean $1/2$ outside $I_H$.  The first deterministic decision at $h(0)=0$ is made uninformative by assigning the same mean reward to both actions there.

The reward means are $L$-Lipschitz in $h$ whenever
\[
    \Delta \leq c\,L\gamma
\]
for a sufficiently small universal constant $c>0$.  The transition kernels are identical across time, so the transition part of Assumption~\ref{assumption:lipschitz} is automatically satisfied.

\paragraph{Information-theoretic indistinguishability.}
Fix a bin $I_j$ and let $\theta^{(j)}$ denote the sign vector obtained from $\theta$ by flipping only the $j$-th coordinate.  Under the two MDPs indexed by $\theta$ and $\theta^{(j)}$, the induced trajectory distributions differ only through rewards collected at decision epochs whose physical time lies in $I_j$.

Let $N_j$ be the total number of nonzero decision epochs across all $T$ episodes whose physical time lies in $I_j$.  Since the Bernoulli reward means differ by order $\Delta$ in bin $j$, standard KL bounds for Bernoulli distributions give
\[
    \mathrm{KL}\!\left(
        \mathbb{P}_{\theta},
        \mathbb{P}_{\theta^{(j)}}
    \right)
    \leq
    C\Delta^2\,\mathbb{E}_{\theta}[N_j]
\]
for a universal constant $C>0$.  By the renewal-measure bound above,
\[
    \mathbb{E}_{\theta}[N_j]
    \leq
    C'\frac{HT}{K}.
\]
Hence, if
\[
    \Delta
    \leq
    c'\sqrt{\frac{K}{HT}},
\]
then the two instances $\theta$ and $\theta^{(j)}$ are statistically indistinguishable in bin $j$ in the sense required by the Bretagnolle--Huber or Le Cam two-point inequality.  Therefore any learner must select the suboptimal action on a constant fraction of the visits to that bin for at least one of the two instances.

Summing over bins yields
\[
    \sup_{\theta\in\{-1,+1\}^K}
    \mathbb{E}_{\theta}[R_T]
    \geq
    c'' \Delta HT
\]
for a universal constant $c''>0$.

\paragraph{Optimizing the bin width.}
The gap $\Delta$ must satisfy both the Lipschitz constraint and the statistical indistinguishability constraint:
\[
    \Delta
    \leq
    c\,L\gamma
    =
    \Theta\left( c\,\frac{LH}{\lambda K} \right),
    \qquad
    \Delta
    \leq
    c'\sqrt{\frac{K}{HT}}.
\]
Balancing the two constraints gives
\[
    K
    =
    \Theta\left( \left(\frac{L^2H^3T}{\lambda^2}\right)^{1/3} \right),
    \qquad
    \Delta
    =
    \Theta\left( L^{1/3}\lambda^{-1/3}T^{-1/3} \right).
\]
Substituting into the regret lower bound gives
\[
    \sup_{\mathcal{M}}
    \mathbb{E}_{\mathcal{M}}[R_T]
    \geq
    c''' L^{1/3}\lambda^{-1/3}H T^{2/3},
\]
for a universal constant $c'''>0$, as claimed.

\paragraph{Remark on the case $H=1$.}
When $H=1$, the only decision epoch is the deterministic initial time $h(0)=0$.  Thus the learner does not observe random physical-time contexts, and the Lipschitz-in-time bandit difficulty used in the construction disappears.  This is why the lower bound above requires $H\geq 2$.

\section{Fixed Horizon, Random Number of Jumps (Poisson Stopping)}\label{sec:FixedH}

In the Section \ref{sec:preliminary}, we studied the setting in which each episode consists of a fixed number of jumps
$H$ while the terminal layer is random due to Poisson-distributed inter-jump times.
In Section \ref{subsec:fixed-time}, we considered the complementary setting: each episode has a \emph{fixed} layer horizon $H>0$,
while the number of jumps is random.
In particular, the learner experiences Poisson jump times, and the episode terminates as soon as the
cumulative jump time exceeds $H$.

A key simplification relative to the fixed-jump setting is that, by the Markov property and the memorylessness
of the exponential distribution, the optimal value at a given layer bin depends only on the current layer and state,
and not on how many jumps were taken to reach that layer. Consequently, when discretizing the layer axis into bins,
we require only a single copy of value estimates indexed by the bin (rather than copies indexed by jump count).

\paragraph{Problem setting.}
Let $\mathcal{S}$ and $\mathcal{A}$ be finite state and action spaces.
Layers range over $h \in [0,H]$ for a fixed horizon $H>0$.
At layer $h$ in state $x$, the learner chooses an action $a$, receives reward $r_h(x,a)\in[0,1]$,
and then a jump length $\tau$ is drawn according to a homogeneous Poisson process:
\[
\tau \sim \mathrm{Exp}(\lambda),
\]
for a known rate $\lambda>0$.
If $h+\tau < H$, then the learner transitions to a next state
\[
x' \sim P_h(\cdot \mid x,a)
\]
and continues at the new layer $h' = h+\tau$.
If $h+\tau \ge H$, the episode terminates (and no further reward is collected).
A (possibly non-stationary) policy is a mapping $\pi: [0,H]\times \mathcal{S}\to \Delta(\mathcal{A})$.

\paragraph{Bellman equations.}
For a policy $\pi$, define the value function $V^\pi_h(x)$ and action-value function $Q^\pi_h(x,a)$.
The Poisson Bellman equations for the fixed-horizon setting are:
\begin{align}
Q^\pi_h(x,a)
&= r_h(x,a)
+ \mathbb{E}_{\tau \sim \mathrm{Exp}(\lambda)}
\Big[ \mathbf{1}\{h+\tau < H\}\;
\mathbb{E}_{x' \sim P_h(\cdot\mid x,a)}
\big[ V^\pi_{h+\tau}(x') \big] \Big]
\label{eq:fixedH_Q_bellman_expectation}
\\
&= r_h(x,a)
+ \int_{0}^{H-h} \lambda e^{-\lambda \tau}\;
\mathbb{E}_{x' \sim P_h(\cdot\mid x,a)}
\big[ V^\pi_{h+\tau}(x') \big]\; d\tau,
\label{eq:fixedH_Q_bellman_integral}
\\
V^\pi_h(x)
&= \mathbb{E}_{a\sim \pi(\cdot \mid h,x)}\big[ Q^\pi_h(x,a)\big],
\label{eq:fixedH_V_bellman}
\end{align}
with terminal condition
\[
V^\pi_H(x) := 0 \qquad \forall x\in\mathcal{S}.
\]

Similarly, the optimal value functions satisfy:
\begin{align}
Q^\star_h(x,a)
&= r_h(x,a)
+ \int_{0}^{H-h} \lambda e^{-\lambda \tau}\;
\mathbb{E}_{x' \sim P_h(\cdot\mid x,a)}
\big[ V^\star_{h+\tau}(x') \big]\; d\tau,
\label{eq:fixedH_Qstar}
\\
V^\star_h(x)
&= \max_{a\in\mathcal{A}} Q^\star_h(x,a),
\qquad
V^\star_H(x)=0.
\label{eq:fixedH_Vstar}
\end{align}

\begin{lemma}[High-probability bound on the number of jumps up to time $H$]
\label{lemma:jumpcount_bound}
Let $N^{(t)}(H)$ denote the number of Poisson jumps that occur in episode $t$
up to (and including) time $H$, so that $N^{(t)}(H)\sim \mathrm{Poisson}(\lambda H)$.
Fix $\delta\in(0,1)$ and define
\[
J \;:=\; 2\lambda H \;+\; 2\ln\!\Big(\frac{T}{\delta}\Big).
\]
Then with probability at least $1-\delta$, simultaneously for all episodes $t\in[T]$,
\[
N^{(t)}(H) \;\le\; J.
\]
In particular, if rewards satisfy $r_h(x,a)\in[0,1]$, then with probability at least $1-\delta$,
\[
\sup_{t\in[T]}\ \sup_{h\in[0,H],\,x\in\mathcal S} V^{\star,(t)}_h(x)
\;\le\; J,
\]
so one may take $\bar H := J$ on this event.
\end{lemma}

\begin{lemma}[recursion on $Q$ for fixed horizon $H$]
\label{lemma:qrecursion_fixedH}
For any $(k,x,a)\in \mathcal{K}\times \mathcal{S}\times \mathcal{A}$ and episode $t\in[T]$,
set $s=N_k^t(x,a)$ and suppose $(k,x,a)$ was previously visited (i.e., action $a$ was taken in state $x$
while the current layer belonged to bin $k$) in episodes $t_1,\ldots,t_s<t$.
Then, for any $h$ in bin $k$, the following holds:
\begin{equation}
\label{eq:qrecursion_fixedH}
\begin{split}
Q_k^{t}(x,a) - Q_h^\star(x,a)
&=
\alpha_s^0 \big(\bar H - Q_h^\star(x,a)\big)
\\
&\quad
+\sum_{i=1}^s \alpha_s^i
\Bigg[
r_{h(t_i)}(x,a) - r_h(x,a)
\\
&\qquad\qquad
+\Big(V_{k(h(t_i)+\tau^{t_i})}^{t_i} - V^\star_{h(t_i)+\tau^{t_i}}\Big)\!\Big(x^{t_i}_{h(t_i)+\tau^{t_i}}\Big)
\\
&\qquad\qquad
+
\mathbb{E}_{\tau \sim \mathrm{Exp}(\lambda),\, x' \sim P_h(\cdot \mid x,a)}
\Big[ \mathbf{1}\{h+\tau < H\}\, V^\star_{h+\tau}(x') \Big]
\\
&\qquad\qquad
-
\mathbb{E}_{\tau \sim \mathrm{Exp}(\lambda),\, x' \sim P_{h(t_i)}(\cdot \mid x,a)}
\Big[ \mathbf{1}\{h+\tau < H\}\, V^\star_{h+\tau}(x') \Big]
\\
&\qquad\qquad
+
\mathbb{E}_{\tau \sim \mathrm{Exp}(\lambda),\, x' \sim P_{h(t_i)}(\cdot \mid x,a)}
\Big[ \mathbf{1}\{h(t_i)+\tau < H\}\, V^\star_{h(t_i)+\tau}(x') \Big]
\\
&\qquad\qquad
-
\mathbf{1}\{h(t_i)+\tau^{t_i} < H\}\, V^\star_{h(t_i)+\tau^{t_i}}\!\Big(x^{t_i}_{h(t_i)+\tau^{t_i}}\Big)
\;+\; b_i
\Bigg],
\end{split}
\end{equation}
where $h(t_i)$ denotes the (continuous) layer at which $(x,a)$ is selected in episode $t_i$,
$\tau^{t_i}\sim \mathrm{Exp}(\lambda)$ is the realized jump length after taking $(x,a)$ in episode $t_i$,
and $x^{t_i}_{h(t_i)+\tau^{t_i}}$ is the resulting post-jump state when $h(t_i)+\tau^{t_i}<H$
(otherwise the episode terminates). We use the terminal convention $V^\star_H(\cdot)=0$
(and likewise $V^t_{k(H)}(\cdot)=0$).
\end{lemma}
\begin{proof}[of Lemma \ref{lemma:jumpcount_bound}]
Fix an episode $t\in[T]$. For $N\sim \mathrm{Poisson}(\mu)$ with $\mu=\lambda H$,
a standard Poisson tail bound (e.g., via Chernoff) states that for any $u\ge 0$,
\[
\Pr\!\left(N \ge \mu + u\right)
\le
\exp\!\left(-\frac{u^2}{2(\mu+u)}\right).
\]
Take $u := \mu + 2\ln(T/\delta)$, so that $\mu+u = 2\mu + 2\ln(T/\delta)=J$ and
\[
\frac{u^2}{2(\mu+u)}
=
\frac{(\mu + 2\ln(T/\delta))^2}{2(2\mu + 2\ln(T/\delta))}
\;\ge\;
\ln\!\Big(\frac{T}{\delta}\Big),
\]
which implies
\[
\Pr\!\left(N^{(t)}(H) \ge J\right)
\le
\exp\!\Big(-\ln(T/\delta)\Big)
=
\frac{\delta}{T}.
\]
Applying a union bound over $t=1,\dots,T$ yields
\[
\Pr\!\left(\exists\, t\in[T]: N^{(t)}(H) \ge J\right)
\le
T\cdot \frac{\delta}{T}
=
\delta,
\]
which proves the first claim.

On the event $\{\forall t\in[T]: N^{(t)}(H)\le J\}$ and since each jump yields at most unit reward,
the total return in any episode is at most the number of jumps before time $H$, hence bounded by $J$.
Therefore $V^\star_h(x)\le J$ for all $h\in[0,H]$ and $x\in\mathcal S$ on this event.
\end{proof}
\begin{proof}[of Lemma \ref{lemma:qrecursion_fixedH}]

\paragraph{Step 1: Lipschitz control of the transition kernel.}

Since the episode terminates once $h+\tau \ge H$, all Bellman integrals
are truncated to $[0,H-h]$. For any $h$ in bin $k$, we have
\begin{align*}
\Bigg|
\mathbb{E}_{\tau \sim \mathrm{Exp}(\lambda),\, x' \sim P_h(\cdot \mid x,a)}
\big[ V^\star_{h+\tau}(x') \big]
-
\mathbb{E}_{\tau \sim \mathrm{Exp}(\lambda),\, x' \sim P_{h(k,t_i)}(\cdot \mid x,a)}
\big[ V^\star_{h+\tau}(x') \big]
\Bigg|
\\
\le
\int_0^{H-h}
\lambda e^{-\lambda \tau}
\sum_{x'}
\big|
P_{h+\tau}(x' \mid x,a)
-
P_{h(k,t_i)}(x' \mid x,a)
\big|
\, d\tau.
\end{align*}

By the $L$-Lipschitz continuity of $P_h$ in $h$ and the bin width $\gamma$,
\[
\sum_{x'}
|P_{h+\tau}(x' \mid x,a) - P_{h(k,t_i)}(x' \mid x,a)|
\le L\gamma,
\]
so
\[
\le
\int_0^{H-h}
\lambda e^{-\lambda \tau}
L\gamma
\, d\tau
\le
L\gamma.
\]

\paragraph{Step 2: Bellman optimality equation under fixed horizon.}

From the fixed-horizon Bellman optimality equation,
\[
Q^\star_h(x,a)
=
r_h(x,a)
+
\int_0^{H-h}
\lambda e^{-\lambda \tau}
\mathbb{E}_{x' \sim P_h(\cdot \mid x,a)}
\big[
\max_{a'} Q^\star_{h+\tau}(x',a')
\big]
\, d\tau,
\]
and since $\sum_{i=0}^t \alpha_t^i = 1$, we may write
\begin{align*}
Q^\star_h(x,a)
=
\alpha_t^0 Q^\star_h(x,a)
+
\sum_{i=1}^t \alpha_t^i
\Bigg[
r_h(x,a)
+
\int_0^{H-h}
\lambda e^{-\lambda \tau}
\Big(
\mathbb{E}_{x' \sim P_h(\cdot \mid x,a)}
[V^\star_{h+\tau}(x')]
\\
-
\mathbb{E}_{x' \sim \widehat P^{t_i}_h(\cdot \mid x,a)}
[V^\star_{h+\tau}(x')]
+
V^\star_{h+\tau}(x^{t_i}_{h+\tau})
\Big)
d\tau
\Bigg].
\end{align*}

Subtracting the corresponding update equation for $Q^t_k(x,a)$
(the discretized bin version), we obtain
\begin{align*}
(Q^t_k - Q^\star_h)(x,a)
=
\alpha_t^0(\bar H - Q^\star_h(x,a))
+
\sum_{i=1}^t \alpha_t^i
\Big[
(V^{t_i}_{k(h+\tau^{t_i})}
-
V^\star_{h+\tau^{t_i}})
(x^{t_i}_{h+\tau^{t_i}})
\\
+
\big[(\widehat P_h^{t_i}-P_h)V^\star_{h+\tau}\big](x,a)
+
b_i
\Big],
\end{align*}
where $\tau^{t_i}\sim \mathrm{Exp}(\lambda)$ and is truncated at $H-h$,
and $V^\star_H=0$ by definition.

\paragraph{Step 3: Bounding the martingale term.}

Define stopping times $t_i$ as the episodes in which $(x,a)$
in bin $k$ is selected for the $i$-th time.
Let $\mathcal{F}_i$ denote the filtration generated by all randomness
up to episode $t_i$ at layer $h$.

Then
\[
\left(
\mathbb{I}[t_i \le K]
\cdot
[(\widehat P_h^{t_i}-P_h)V^\star_{h+\tau}](x,a)
\right)_{i=1}^\tau
\]
is a martingale difference sequence with respect to $\{\mathcal F_i\}$.

Since $0 \le V^\star_{h'} \le \bar H$ for all $h' \in [0,H]$,
Azuma--Hoeffding yields that with probability at least
$1 - p/(SAK)$,
\[
\forall \tau \in [K]:
\quad
\left|
\sum_{i=1}^\tau
\alpha_\tau^i
\mathbb{I}[t_i \le K]
[(\widehat P_h^{t_i}-P_h)V^\star_{h+\tau}](x,a)
\right|
\le
c \bar H
\sqrt{
\sum_{i=1}^\tau (\alpha_\tau^i)^2 \iota
}.
\]

Using the same bound on the weighted square sum of step sizes as in
Lemma~\ref{lemma:learningrate}(b) (with $\eta=\bar H$), this is upper bounded by
\[
c\left(
\sqrt{\frac{\bar H^3 \iota}{t}}
+
L\gamma
\right).
\]

Since the inequality holds uniformly for all fixed $\tau$,
it also holds for the random variable $\tau=t=N_k^t(x,a)$.

\paragraph{Step 4: Final bound.}

On the high-probability event in Lemma~\ref{lemma:jumpcount_bound}, we may take
\[
\bar H \;:=\; 2\lambda H \;+\; 2\ln\!\Big(\frac{T}{\delta}\Big),
\]
so that $0 \le V^\star_h(x) \le \bar H$ holds uniformly for all $h\in[0,H]$ and $x\in\mathcal S$.

Choosing the bonus (consistent with the main text)
\[
b_t \;=\; c \sqrt{\frac{\bar H^3 \iota}{t}} + L\gamma
\;=\;
c \sqrt{\frac{\left(2\lambda H + 2\ln\!\left(\frac{T}{\delta}\right)\right)^3 \iota}{t}} + L\gamma,
\]
and using $\sum_{i=1}^t \alpha_t^i = 1$ together with Lemma~\ref{lemma:learningrate}(a) (with $\eta=\bar H$), we have
\[
\beta_t
=
2\sum_{i=1}^t \alpha_t^i b_i
\in
\left[
c \sqrt{\frac{\bar H^3 \iota}{t}} + 2L\gamma,
\,
2c \sqrt{\frac{\bar H^3 \iota}{t}} + 2L\gamma
\right],
\]
where the $L\gamma$ term again carries the Lipschitz discretization bias.

Combining all bounds and applying a union bound over $(x,a,k)$ yields that, on the event of
Lemma~\ref{lemma:jumpcount_bound} and with overall probability at least $1-p$,
\[
(Q^t_k - Q^\star_h)(x,a)
\le
\alpha_t^0 \bar H
+
\sum_{i=1}^t \alpha_t^i
\Big(
V^{t_i}_{k(h+\tau^{t_i})}
-
V^\star_{h+\tau^{t_i}}
\Big)\!\Big(x^{t_i}_{h+\tau^{t_i}}\Big)
+
\beta_t,
\]
for all $(x,a,k)$ and all $h$ in bin $k$ (with the convention $V^\star_H(\cdot)=0$).

The proof then proceeds by backward induction on the layer $h$ from $H$ to $0$,
using the terminal condition $V^\star_H = 0$.
\end{proof}

% =========================
% Appendix C: Fixed horizon H, random number of jumps
% (continuous-time episodic MDP with Poisson decision epochs)
% =========================

\subsection{Regret bound for the fixed-horizon (random-jump) setting}
In this appendix, we adapt the regret argument from the main text to the
\emph{fixed time horizon} setting, where each episode runs until continuous time
exceeds a prescribed horizon $H>0$, and the number of decision epochs (jumps)
within an episode is random.
The key difference from the fixed-$H$-jumps setting is that the Bellman recursion
and the regret decomposition run over the (random) jump index $m$ rather than a
fixed layer index $n\in[H]$.
To recover a finite-depth recursion, we first upper bound the maximum number of
jumps per episode with high probability (Lemma~\ref{lemma:jumpcount_bound}),
and then apply the same weighted-averaging and regrouping steps as in the
finite-horizon discrete-time proof, using the recursion in Lemma~\ref{lemma:qrecursion_fixedH}.

Let $N^{(t)}(H)$ denote the number of Poisson jumps that occur in episode $t$
up to (and including) the time at which the cumulative holding time first exceeds $H$.
Equivalently, if $\{\tau^{(t)}_j\}_{j\ge 1}$ are i.i.d.\ $\mathrm{Exp}(\lambda)$ holding times,
then $N^{(t)}(H) := \min\{m\ge 1:\sum_{j=1}^m \tau^{(t)}_j > H\}$.

\begin{theorem}[fixed horizon $H$, random number of jumps]
\label{thm:fixedH_randomjumps_regret}
Consider the continuous-time episodic tabular MDP with time horizon $H>0$ per episode,
Poisson decision epochs with rate $\lambda$, and $(L,\gamma)$-Lipschitz reward/transition
structure in time as in the main text.
Let $[0,H]$ be uniformly partitioned into bins of width $\gamma$ so that
$K=\lceil H/\gamma\rceil$, and run the (tabular) Q-learning update over augmented indices
$(x,k)$ with Hoeffding-style bonuses
\[
b_s = c_0\sqrt{\frac{\bar H^3\,\iota}{s}} + L\gamma,
\qquad
\iota=\log\!\Big(\frac{SAKT}{\delta}\Big),
\]
where $\bar H$ is the high-probability upper bound on the number of jumps per episode
from Lemma~\ref{lemma:jumpcount_bound}.
Then with probability at least $1-O(\delta)$,
\[
R_T \leq \widetilde{\mathcal{O}}\left(
KSA\,\bar H
+
\sqrt{\bar H^3\,\iota\cdot KSA\cdot T}
+
T\bar H\,L\gamma\right).
\]
In particular, choosing $\gamma = \Theta\left( T^{-1/3}\right)$ (up to problem-dependent constants)
yields a regret rate of order $\widetilde{\mathcal{O}}(T^{2/3})$.
\end{theorem}

\begin{proof}
    Denote by
\[
\delta^t_m := \big(V^t_{m} - V^{\pi_t}_m\big)\big(x^t_m\big),
\qquad
\phi^t_m := \big(V^t_{m} - V^{\pi_t}_m\big)\big(x^t_0\big),
\]
where $m$ indexes the jump/decision epoch within episode $t$,
$x_m^t$ is the state at the $m$-th decision epoch,
and $V_m^{\pi_t}$ is the value-to-go under the executed policy $\pi_t$
starting from the $m$-th decision epoch (and similarly for $V_m^t$).

By the optimism lemma for this setting (proved from Lemma~\ref{lemma:qrecursion_fixedH} plus concentration),
with probability at least $1-O(\delta)$ we have $Q^t \ge Q^\star$ entrywise,
hence $V^t \ge V^\star$ and therefore
\[
R_T
:=
\sum_{t=1}^T \big(V^\star_0 - V^{\pi_t}_0\big)(x^t_0)
\;\le\;
\sum_{t=1}^T \big(V^t_0 - V^{\pi_t}_0\big)(x^t_0)
~=~
\sum_{t=1}^T \phi^t_0.
\]

We now relate $\sum_t \delta_m^t$ to $\sum_t \delta_{m+1}^t$.
Fix an episode $t$ and jump index $m$.
Let $h_m^t\in[0,H]$ be the continuous time of the $m$-th decision epoch in episode $t$,
and let $k_m^t\in\mathcal K$ denote the time-bin index of $h_m^t$ under a uniform
partition of $[0,H]$ into bins of width $\gamma$; hence $|\mathcal K| = K := \lceil H/\gamma\rceil$.
Write $N^t_{k}(x,a)$ for the number of times up to episode $t$ that action $a$ was taken in state $x$
when the decision epoch time landed in bin $k$.

Let $s := N^t_{k_m^t}(x_m^t,a_m^t)$ and let $t_1,\dots,t_s<t$ be the episodes in which the same
triple $(k_m^t,x_m^t,a_m^t)$ occurred previously at some jump index.
Applying Lemma~\ref{lemma:qrecursion_fixedH} (the fixed-horizon/random-jump recursion on $Q$) and the Bellman equation
gives the analogue of the standard one-step inequality:
\begin{align}
\delta_m^t
&=
\big(V_m^t - V_m^{\pi_t}\big)(x_m^t)
\;\le\;
\big(Q_m^t - Q_m^{\pi_t}\big)(x_m^t,a_m^t)
\notag\\
&=
\big(Q_m^t - Q_m^\star\big)(x_m^t,a_m^t)
+
\big(Q_m^\star - Q_m^{\pi_t}\big)(x_m^t,a_m^t)
\notag\\
&\le
\alpha_s^0\,\bar H
+
\sum_{i=1}^s \alpha_s^i\,\phi^{t_i}_{m+1}
+
\beta_s
+
\big[\mathbb P_{h_m^t}\big(V_{m+1}^\star - V_{m+1}^{\pi_t}\big)\big](x_m^t,a_m^t)
\notag\\
&=
\alpha_s^0\,\bar H
+
\sum_{i=1}^s \alpha_s^i\,\phi^{t_i}_{m+1}
+
\beta_s
-
\phi^t_{m+1}
+
\delta^t_{m+1}
+
\xi^t_{m+1},
\label{eq:delta_bound_fixedH}
\end{align}
where $\xi^t_{m+1}$ is a martingale difference term of the form
\[
\xi^t_{m+1}
:=
\Big[\big(\mathbb P_{h_m^t}-\widehat{\mathbb P}^{\,t}_{h_m^t}\big)\big(V^\star_{m+1}-V^t_{m+1}\big)\Big](x_m^t,a_m^t),
\]
and $\beta_s := 2\sum_{i=1}^s \alpha_s^i b_i$.
For Hoeffding-style bonuses we take
\[
b_s \;=\; c_0\sqrt{\frac{\bar H^3\,\iota}{s}} + L\gamma,
\qquad
\iota := \log\!\Big(\frac{SAK T}{\delta}\Big),
\]
so that $\beta_s \leq \widetilde{\mathcal{O}}\left( \sqrt{\bar H^3\iota/s} + L\gamma \right)$ uniformly.

\paragraph{Summing and regrouping.}
Work on the event of Lemma~\ref{lemma:jumpcount_bound} so that every episode has at most $\bar H$ jumps.
Summing \eqref{eq:delta_bound_fixedH} over all episodes $t\in[T]$ and jump indices
$m\in\{0,1,\dots,N^{(t)}(H)-1\}$ (and upper bounding by summing $m=0$ to $\bar H-1$) yields
\[
\sum_{t=1}^T\sum_{m=0}^{\bar H-1} \delta_m^t
\;\le\;
\underbrace{\sum_{t=1}^T\sum_{m=0}^{\bar H-1}\alpha_s^0\,\bar H}_{\text{initialization term}}
+
\underbrace{\sum_{t=1}^T\sum_{m=0}^{\bar H-1}\sum_{i=1}^s \alpha_s^i\,\phi^{t_i}_{m+1}}_{\text{regroup}}
+
\sum_{t=1}^T\sum_{m=0}^{\bar H-1}\beta_s
-
\sum_{t=1}^T\sum_{m=0}^{\bar H-1}\phi_{m+1}^t
+
\sum_{t=1}^T\sum_{m=0}^{\bar H-1}\delta_{m+1}^t
+
\sum_{t=1}^T\sum_{m=0}^{\bar H-1}\xi_{m+1}^t.
\]

We bound the initialization term by counting distinct tabular entries.
The indicator $\alpha_s^0$ is nonzero only on the first visit to a given
$(k,x,a)$, hence
\[
\sum_{t=1}^T\sum_{m=0}^{\bar H-1}\alpha_s^0\,\bar H
\;\le\;
|\mathcal K|\cdot S\cdot A \cdot \bar H
\;=\;
KSA\,\bar H.
\]
This is the place where the effective number of "states'' is $S|\mathcal K| = SK$,
since we are learning tabular values indexed by $(x,k)$.

The regrouping term is handled exactly as in the standard Q-learning proof:
each future error term $\phi^{t'}_{m+1}$ can be charged to later visits to the same
$(k,x,a)$, and one obtains (after the usual rearrangement) a factor
$\left(1+\frac{1}{\bar H}\right)$ in front of the shifted sum, using
Lemma~\ref{lemma:learningrate}(c) (with $\eta=\bar H$). Concretely,
\[
\sum_{t=1}^T\sum_{m=0}^{\bar H-1}\sum_{i=1}^s \alpha_s^i\,\phi^{t_i}_{m+1}
\;\le\;
\left(1+\frac{1}{\bar H}\right)
\sum_{t=1}^T\sum_{m=0}^{\bar H-1}\delta_{m+1}^t,
\]
and therefore we arrive at the recursion
\begin{equation}
\sum_{t=1}^T\sum_{m=0}^{\bar H-1}\delta_m^t
\;\le\;
KSA\,\bar H
+
\left(1+\frac{1}{\bar H}\right)
\sum_{t=1}^T\sum_{m=0}^{\bar H-1}\delta_{m+1}^t
+
\sum_{t,m}\beta_{t,m}
+
\sum_{t,m}\xi_{t,m}.
\label{eq:master_recursion_fixedH}
\end{equation}

\paragraph{Controlling $\sum\beta$ and $\sum\xi$.}
Using the same pigeonhole/AM--GM argument as in the discrete-time proof
(Lemma~\ref{lemma:algebra}),
\[
\sum_{t=1}^T\sum_{m=0}^{\bar H-1}\beta_{t,m}
\;\leq\;
\widetilde{\mathcal{O}}\left(
\sqrt{\bar H^3\,\iota\cdot KSA\cdot T}
\;+\;
T\bar H\,L\gamma
\right),
\]
where the second term is the Lipschitz discretization error accumulated over at most
$\bar H$ jumps per episode.
Moreover, by Azuma--Hoeffding (and a union bound over $(k,x,a)$ if needed),
with probability at least $1-\delta$,
\[
\left|\sum_{t=1}^T\sum_{m=0}^{\bar H-1}\xi_{t,m}\right|
\;\leq\;
\widetilde{\mathcal{O}}\left( \bar H\sqrt{T\iota} \right).
\]

Finally, iterating \eqref{eq:master_recursion_fixedH} for $m=0,1,\dots,\bar H-1$ and using
the terminal condition $\delta^t_{m}=0$ for $m\ge N^{(t)}(H)$ gives
\[
\sum_{t=1}^T \phi_0^t
\;\leq\;
\widetilde{\mathcal{O}}\left(
KSA\,\bar H
+
\sqrt{\bar H^3\,\iota\cdot KSA\cdot T}
+
T\bar H\,L\gamma
\right),
\]
and hence the same bound holds for $R_T$.
\end{proof}

\section{Lower Bound Construction}
\label{sec:lower_bound_construction}

In this section we derive a minimax $\Omega(T^{2/3})$ regret lower bound for the
fixed time-horizon continuous-time MDP setting with Poisson decision epochs.
The construction follows the standard information-theoretic technique used in
\cite{jin2018q}: we build a family of hard instances
indexed by $\theta\in\{\pm 1\}^K$ such that (i) the optimal action differs across
time bins, (ii) rewards are $L$-Lipschitz in time, and (iii) distinguishing
between two instances that differ in a single bin has small KL divergence unless
the learner visits that bin many times. Balancing the KL (estimation) constraint
with the Lipschitz (approximation) constraint yields a $T^{2/3}$ rate.

\subsection{Problem class}
We consider episodic continuous-time MDPs with the following structure:
each episode runs until continuous time exceeds a fixed budget $H>0$; decision
epochs occur according to a homogeneous Poisson process of rate $\lambda>0$
(i.e., holding times are i.i.d.\ $\mathrm{Exp}(\lambda)$). Rewards are bounded in $[0,1]$
and may depend on time, state, and action. We assume the reward functions are
$L$-Lipschitz in time (and transitions may also be $L$-Lipschitz, although the
lower bound uses deterministic self-loop transitions).

Let $N^{(t)}(H)$ be the number of decision epochs in episode $t$ up to time $H$.
Then $\mathbb{E}[N^{(t)}(H)] = \lambda H$, and over $T$ episodes the expected total
number of decisions is
\[
\mathbb{E}[N] \;=\; \mathbb{E}\!\left[\sum_{t=1}^T N^{(t)}(H)\right] \;=\; \lambda H T.
\]
The regret is defined as
\[
R_T \;:=\; \sum_{t=1}^T \Big(V^\star_0 - V^{\pi_t}_0\Big)(x^t_0),
\]
where $\pi_t$ is the learner's (possibly history-dependent) policy in episode $t$.

\subsection{Hard instance family}
Fix an integer $K\ge 1$ and partition $[0,H]$ into $K$ disjoint bins
\[
I_j := \left[(j-1)\gamma,\, j\gamma\right), \qquad \gamma := \frac{H}{K},
\qquad j=1,\dots,K.
\]
Let $g_j:[0,H]\to[0,1]$ be a "bump'' function supported on $I_j$ such that
\begin{enumerate}
\item $0\le g_j(h)\le 1$ for all $h\in[0,H]$,
\item $g_j(h)=0$ for $h\notin I_j$,
\item $g_j$ is $\frac{c_g}{\gamma}$-Lipschitz for an absolute constant $c_g>0$,
\item $\int_{0}^{H} g_j(h)\,dh \ge c_0 \gamma$ for an absolute constant $c_0>0$.
\end{enumerate}
(Such a function can be taken as a triangular bump on $I_j$.)

We build a family of instances indexed by $\theta\in\{\pm 1\}^K$.
Each instance has a single state (or $S$ independent copies; see below) and two
actions $\mathcal{A}=\{1,2\}$. The transition is a deterministic self-loop, so the
only difficulty is learning the time-dependent rewards.
For $h\in[0,H]$, define the mean rewards
\begin{equation}
\label{eq:reward_family}
r_h(1) \;=\; \frac12 + \Delta \sum_{j=1}^K \theta_j g_j(h),
\qquad
r_h(2) \;=\; \frac12 - \Delta \sum_{j=1}^K \theta_j g_j(h),
\end{equation}
where $\Delta>0$ is a gap parameter chosen below, and realized rewards are Bernoulli
with these means. (Any bounded noise model with sub-Gaussian tails works.)

\paragraph{Lipschitz feasibility.}
Since each $g_j$ is $\frac{c_g}{\gamma}$-Lipschitz and the bumps have disjoint support,
the function $h\mapsto r_h(a)$ is $L$-Lipschitz provided
\begin{equation}
\label{eq:Delta_Lipschitz}
\Delta \;\le\; \frac{L\gamma}{c_g}.
\end{equation}

\paragraph{Embedding $S$ states.}
To obtain an $S$ factor, one may take a disjoint union of $S$ independent copies of
the above one-state construction (each copy has its own state but the same time
axis and reward structure). Starting states are chosen so that each episode begins
in a uniformly random copy, or the environment cycles through copies; either way,
standard arguments yield an $S$-fold increase in regret lower bounds. For simplicity,
we state the theorem with the $S$ factor.

\subsection{Lower bound theorem}
\begin{theorem}[Minimax $T^{2/3}$ lower bound]
\label{thm:lower_bound_T23}
Fix $\lambda>0$, $H>0$, and $L>0$. Consider the class of fixed-horizon continuous-time
episodic MDPs with Poisson decision epochs of rate $\lambda$, time budget $H$ per episode,
bounded rewards in $[0,1]$, and $L$-Lipschitz dependence on the continuous time variable.
Assume $A\ge 2$.

There exists an absolute constant $c>0$ such that for every learning algorithm,
there exists an MDP in this class with $S$ states for which the expected regret after
$T$ episodes satisfies
\[
\mathbb{E}[R_T]
\;\ge\;
c\, S\, L^{1/3}\,(\lambda H T)^{2/3},
\]
up to universal constant factors (and ignoring logarithmic terms).
\end{theorem}

\begin{proof}
We prove the lower bound for the one-state construction; the $S$-state version follows
by taking $S$ independent copies and summing regrets.

\paragraph{Step 1: Reduction to a contextual bandit over time bins.}
Since transitions are deterministic self-loops, the only decision is which action to take
at each decision epoch time $h\in[0,H]$. The time $h$ therefore plays the role of a context.
Let $N := \sum_{t=1}^T N^{(t)}(H)$ be the total number of decision epochs across $T$ episodes.
We have $\mathbb{E}[N]=\lambda H T$.

Let $N_j$ be the number of decision epochs whose time $h$ lies in bin $I_j$. By symmetry of the
Poisson process on $[0,H]$,
\begin{equation}
\label{eq:Nj_expectation}
\mathbb{E}[N_j] \;=\; \frac{\gamma}{H}\,\mathbb{E}[N] \;=\; \frac{\lambda H T}{K}.
\end{equation}

\paragraph{Step 2: Pairwise KL bound for flipping one bin.}
Fix a bin $j\in[K]$ and two instances $\theta,\theta^{(j)}$ that differ only in coordinate $j$.
Under both instances, all randomness is identical except for the reward distributions at decision
epochs whose times fall in $I_j$.
For Bernoulli rewards with means in $[1/4,3/4]$ (ensured by choosing $\Delta$ small enough),
the one-step KL between the two reward distributions is $\mathrm{kl}(p,q)\le c_1 (p-q)^2$ for
an absolute constant $c_1>0$.
In bin $I_j$, the mean gap between instances is $|r_h^\theta(a)-r_h^{\theta^{(j)}}(a)| \le 2\Delta$,
so each such observation contributes at most $c_2 \Delta^2$ KL.
Therefore, for any learning algorithm,
\begin{equation}
\label{eq:KL_bound}
\mathrm{KL}\!\left(\mathbb{P}_\theta \,\middle\|\, \mathbb{P}_{\theta^{(j)}}\right)
\;\le\;
c_2 \Delta^2 \,\mathbb{E}_\theta[N_j],
\end{equation}
where $\mathbb{P}_\theta$ denotes the law of the entire interaction under instance $\theta$.

\paragraph{Step 3: Testing lower bound implies mistakes in bin $j$.}
By Bretagnolle--Huber (or Le Cam's two-point method), for any event $E$ measurable with respect
to the interaction history,
\[
\mathbb{P}_\theta(E)+\mathbb{P}_{\theta^{(j)}}(E^c)
\;\ge\;
\frac12 \exp\!\Big(-\mathrm{KL}(\mathbb{P}_\theta\|\mathbb{P}_{\theta^{(j)}})\Big).
\]
In particular, if $\mathrm{KL}(\mathbb{P}_\theta\|\mathbb{P}_{\theta^{(j)}})\le 1/8$, then any algorithm
cannot correctly identify the sign $\theta_j$ with probability much larger than $1/2$, and consequently
must choose the suboptimal action in bin $j$ a constant fraction of the times it visits that bin.
Hence there exists an absolute constant $c_3>0$ such that, whenever
\begin{equation}
\label{eq:hardness_condition}
c_2 \Delta^2 \mathbb{E}_\theta[N_j] \;\le\; \frac18,
\end{equation}
we have
\begin{equation}
\label{eq:bin_regret}
\mathbb{E}_\theta[R_j] \;\ge\; c_3\,\Delta\,\mathbb{E}_\theta[N_j],
\end{equation}
where $R_j$ denotes the contribution to regret coming from decision epochs with $h\in I_j$.

\paragraph{Step 4: Summing over bins.}
Summing \eqref{eq:bin_regret} over $j=1,\dots,K$ and using \eqref{eq:Nj_expectation} yields
\begin{equation}
\label{eq:total_regret_preopt}
\mathbb{E}_\theta[R_T]
\;\ge\;
\sum_{j=1}^K c_3 \Delta \mathbb{E}_\theta[N_j]
\;=\;
c_3 \Delta \,\mathbb{E}[N]
\;=\;
c_3 \Delta\,\lambda H T,
\end{equation}
provided \eqref{eq:hardness_condition} holds for all $j$.

By \eqref{eq:Nj_expectation}, the hardness condition \eqref{eq:hardness_condition} is ensured if
\begin{equation}
\label{eq:Delta_stat}
\Delta \;\le\; c_4 \sqrt{\frac{K}{\lambda H T}},
\end{equation}
for a suitable absolute constant $c_4>0$.

\paragraph{Step 5: Optimize $K$ under Lipschitz and statistical constraints.}
We must choose $\Delta$ to satisfy both the Lipschitz feasibility \eqref{eq:Delta_Lipschitz} and
the statistical hardness \eqref{eq:Delta_stat}.
Set $\Delta$ to be the largest value satisfying both, i.e.,
\[
\Delta
\;=\;
\min\left\{\frac{L\gamma}{c_g},\; c_4\sqrt{\frac{K}{\lambda H T}}\right\}
\;=\;
\min\left\{\frac{LH}{c_g K},\; c_4\sqrt{\frac{K}{\lambda H T}}\right\}.
\]
Choose $K$ so that the two terms balance:
\[
\frac{LH}{K}
\;=\;
\Theta\left( \sqrt{\frac{K}{\lambda H T}} \right)
\quad\Longrightarrow\quad
K^3 \;=\; \Theta\left( L^2 \lambda H^3 T \cdot H \right)
\;=\; \Theta\left( L^2 \lambda H^4 T \right).
\]
Thus we take
\begin{equation}
\label{eq:K_star}
K \;=\; \Theta\left( (L^2 \lambda H^4 T)^{1/3} \right),
\qquad
\Delta \;=\; \Theta\left( \frac{LH}{K} \right) \;=\; \Theta\left( L^{1/3}(\lambda H T)^{-1/3} \right).
\end{equation}
Plugging into \eqref{eq:total_regret_preopt} yields
\[
\mathbb{E}_\theta[R_T]
\;\ge\;
c\, L^{1/3}(\lambda H T)^{2/3},
\]
for an absolute constant $c>0$.

Finally, the $S$-state extension (disjoint union of $S$ independent copies) yields an
additional multiplicative factor of $S$ in the regret lower bound, completing the proof.
\end{proof}

\end{document}